%% file: RDP_Sampled_Shuffle.tex
\documentclass{article}

\usepackage{fullpage}
\usepackage{url,color,cite}
\usepackage{caption}
\usepackage{amsmath,amsthm,amsfonts,amssymb}
\usepackage{mathtools}

\DeclarePairedDelimiter{\floor}{\lfloor}{\rfloor}
\usepackage{float}
\usepackage{epstopdf}
\usepackage{nicefrac}
\usepackage{algorithm}
\usepackage[noend]{algpseudocode}
\usepackage{pifont}
\usepackage{subcaption}
\usepackage{sidecap}
\usepackage{multirow}
\usepackage{multicol}
\usepackage{bbm}
\usepackage{enumitem}

\newtheorem{lemma}{Lemma}
\newtheorem*{lemma*}{Lemma}
\newtheorem{theorem}{Theorem}
\newtheorem*{theorem*}{Theorem}
\theoremstyle{definition}
\newtheorem{defn}{Definition}
\theoremstyle{definition}

\newtheorem{remark}{Remark}
\theoremstyle{definition}
\newtheorem{corollary}{Corollary}
\newtheorem*{corollary*}{Corollary}
\theoremstyle{definition}

\newtheorem*{claim*}{Claim}

\newcommand{\same}{\text{same}}

\newcommand{\bbE}{\mathbb{E}}
\newcommand{\bbN}{\mathbb{N}}
\newcommand{\bbR}{\mathbb{R}}

\newcommand{\calA}{\mathcal{A}}
\newcommand{\calB}{\mathcal{B}}
\newcommand{\calC}{\mathcal{C}}
\newcommand{\calD}{\mathcal{D}}

\newcommand{\calH}{\mathcal{H}}
\newcommand{\calM}{\mathcal{M}}

\newcommand{\calO}{\mathcal{O}}
\newcommand{\calP}{\mathcal{P}}

\newcommand{\calR}{\mathcal{R}}
\newcommand{\calS}{\mathcal{S}}

\newcommand{\calU}{\mathcal{U}}

\newcommand{\calX}{\mathcal{X}}
\newcommand{\calY}{\mathcal{Y}}

\renewcommand{\(}{\left(}
\renewcommand{\)}{\right)}

\newcommand{\eps}{\epsilon}

\newcommand{\samp}{\mathrm{samp}}

\newcommand{\bh}{\boldsymbol{h}}

\newcommand{\bu}{\boldsymbol{u}}

\newcommand{\bx}{\boldsymbol{x}}

\newcommand{\bp}{\boldsymbol{p}}

\usepackage{microtype}
\usepackage{graphicx}
\usepackage{booktabs} 
\usepackage{hyperref}
\usepackage{cleveref}

\usepackage{lipsum}

\title{Renyi Differential Privacy of the Subsampled Shuffle Model in Distributed Learning}

\author{Antonious M. Girgis, Deepesh Data, Suhas Diggavi}

\date{}

\begin{document}

\maketitle

\input{abstract}

{\allowdisplaybreaks

\input{introduction}

\input{preliminaries}

\input{main_results}

\input{numerics}

\input{general_case}

\input{ternary_DP_shuffle}

\input{app_lower_bound}
\input{app_optimization_performance}

\bibliographystyle{alpha}
\bibliography{RDPRefs}

\newpage
\appendix
\input{app_preliminaries}

\input{app_reduce_special_case}
\input{app_ternary_special_case_DD}

\input{app_sampling}

}
\end{document}

%% file: abstract.tex
\begin{abstract}
We study privacy in a distributed learning framework, where clients
collaboratively build a learning model iteratively through
interactions with a server from whom we need privacy. Motivated by
stochastic optimization and the federated learning (FL) paradigm, we
focus on the case where a small fraction of data samples are randomly
sub-sampled in each round to participate in the learning process,
which also enables privacy amplification.  To obtain even stronger
local privacy guarantees, we study this in the shuffle privacy model,
where each client randomizes its response using a local differentially
private (LDP) mechanism and the server only receives a random
permutation (shuffle) of the clients' responses without their
association to each client. The principal result of this paper is a
privacy-optimization performance trade-off for discrete randomization
mechanisms in this sub-sampled shuffle privacy model. This is enabled
through a new theoretical technique to analyze the Renyi Differential
Privacy (RDP) of the sub-sampled shuffle model.  We numerically
demonstrate that, for important regimes, with composition our bound
yields significant improvement in privacy guarantee over the
state-of-the-art approximate Differential Privacy (DP) guarantee (with
strong composition) for sub-sampled shuffled models. We also
demonstrate numerically significant improvement in privacy-learning
performance operating point using real data sets.
\end{abstract}

%% file: introduction.tex
\section{Introduction}\label{sec:introduction}

As learning moves towards the edge, there is a need to collaborate to
build learning models\footnote{This is because no client has access to
  enough data to build rich learning models locally and we do not want to
  directly share local data.}, such as in federated learning
\cite{konevcny2016federated,yang2019federated,kairouz2019advances}. In
this framework, the collaboration is typically mediated by a
server. In particular, we want to collaboratively build a learning model by solving 
an empirical risk minimization (ERM) problem (see \eqref{eq:problem-formulation} in Section \ref{sec:prelims_prob-form}). To obtain a model parametrized by $\theta$ using
ERM, the commonly used mechanism is Stochastic Gradient
Descent (SGD) \cite{bottou2010large}. However, one needs to solve this
while enabling strong privacy guarantees on local data from the
server, while also obtaining good learning performance, \emph{i.e.,} a
suitable privacy-learning performance operating point.

Differential privacy (DP)~\cite{Calibrating_DP06} is the gold standard
notion of data privacy that gives a rigorous framework through
quantifying the information leakage about individual training data
points from the observed interactions.  Though DP was originally
proposed in a framework where data resides centrally
\cite{Calibrating_DP06}, for distributed learning the more
appropriate notion is of local differential privacy
(LDP)~\cite{kasiviswanathan2011can,duchi2013local}. Here, each client
randomizes its interactions with the server from whom the data is to
be kept private (\emph{e.g.,} see industrial implementations
~\cite{erlingsson2014rappor,greenberg2016apple,microsoft}).  However,
LDP mechanisms suffer from poor performance in comparison with the
central DP mechanisms~\cite{duchi2013local,kasiviswanathan2011can,kairouz2016discrete}. To
overcome this, a new privacy framework using anonymization has been
proposed in the so-called \emph{shuffled model}
\cite{erlingsson2019amplification,ghazi2019power,balle2019improved,ghazi2019scalable,balle2019differentially,cheu2019distributed,balle2019privacy,balle2020private}.
In the shuffled model, each client sends her private message to a
secure shuffler that randomly permutes all the received messages
before forwarding them to the server. This model enables significantly
better privacy-utility performance by amplifying DP through this shuffling. %this anonymization. 
Therefore, in this paper we consider the shuffle privacy framework for
distributed learning.

In solving \eqref{eq:problem-formulation} using (distributed) gradient descent, each
exchange leaks information about the local data, but we need as many
steps as possible to obtain a good model; setting up the tension
between privacy and performance. The goal is to obtain as many such
interactions as possible for a given privacy budget. This is
quantified through analyzing the privacy of the composition of privacy
mechanisms.  Abadi \emph{et al.}  \cite{abadi2016deep} developed a
framework for tighter analysis of such compositions, and this was
later reformulated in terms of Renyi Differential Privacy
(RDP)~\cite{mironov2017renyi}, and mapping this back to DP guarantee
\cite{mironov2019r}. Therefore, studying RDP is important to obtaining
strong composition privacy results, and is the focus of this paper.

\begin{figure}   
\centering
    \includegraphics[scale=0.3]{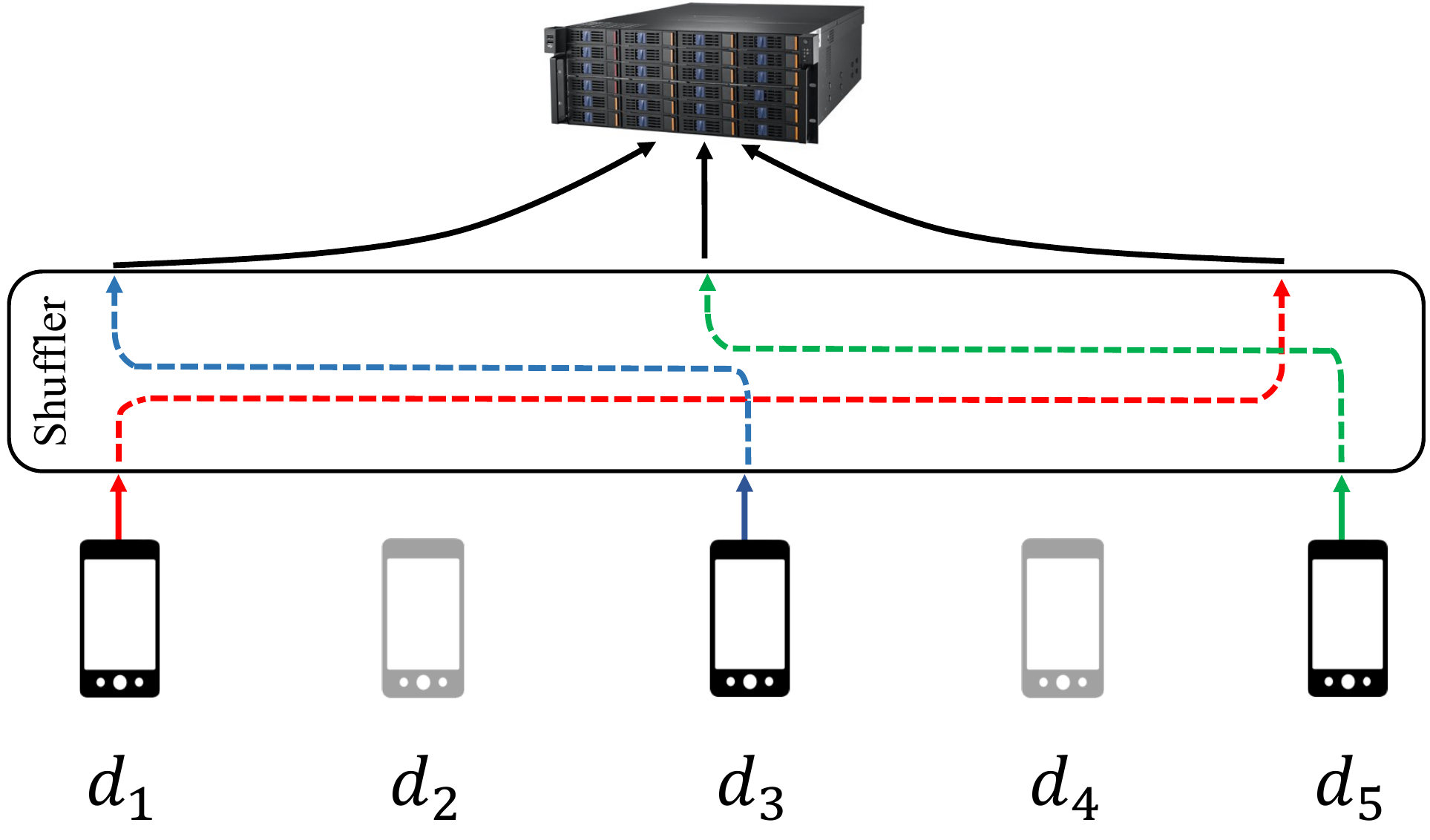}
    \captionof{figure}{An iteration from the CLDP-SGD Algorithm, where $3$ clients are randomly chosen at each iteration. Each client sends the private gradient $\calR_p\left(g_t(d_i)\right)$ to the shuffler that randomly permutes the gradients before passing them to the server.}
    \label{fig:Setup}
\end{figure}
  
In distributed (and federated) learning, a fraction of the data
samples are sampled; for example, with random client participation and
stochastic gradient descent (SGD), which can be written as
\[
\theta_{t+1} \leftarrow \theta_t - \eta_t \frac{1}{|\mathcal{I}|}\sum_{i\in\mathcal{I}} \mathcal{R}(\nabla f_i(\theta_{t})),
\]
where $\mathcal{R}$ is the local randomization mechanism and
$\mathcal{I}$ are the indices of the sampled data. This is a
subsampled mechanism that enables another privacy amplification
opportunity; which, in several cases, is shown to yield a privacy
advantage proportional to the subsampling rate; see
\cite{kasiviswanathan2011can,Jonathan2017sampling}.  The central
technical question addressed in this paper is how to analyze the RDP
of an arbitrary discrete mechanism for the subsampled shuffle privacy
model. This enables us to answer the overall question posed in this
paper, which is an achievable privacy-learning performance trade-off
point for solving \eqref{eq:problem-formulation} in the shuffled privacy model for
distributed learning (see Figure \ref{fig:Setup}). Our contributions are:

\begin{itemize}[leftmargin=*]
\item We analyze the RDP of subsampled mechanisms in the shuffle
  framework by developing a novel bound applicable to any discrete
  $\epsilon_0$-LDP mechanism as a function of the RDP order $\lambda$,
  subsampling rate $\gamma$, the LDP parameter $\epsilon_0$, and the
  number of clients $n$; see Theorem~\ref{thm:general_case}. The bound
  is explicit and amenable to numerics, including all
  constants.\footnote{As emphasized in \cite{wang2019subsampled}, ``in
    differential privacy, constants matter''.} Furthermore, the
  bounds are valid for generic LDP mechanisms and {\em all} parameter
  regimes.\footnote{Some of the best known approximate DP bounds for
    the shuffle model~\cite{balle2019privacy,feldman2020hiding} are restricted to certain parameter regimes in
    terms of $n, \delta, \epsilon_0$, etc.}  We also provide a
  lower bound for the RDP in Theorem \ref{thm:lower_bound}. 
  We prove our upper bound (Theorem~\ref{thm:general_case}) using the following novel analysis techniques:
  First, we reduce the problem of computing the RDP of sub-sampled shuffle 
  mechanisms to the problem of computing ternary $|\chi|^{\alpha}$-DP \cite{wang2019subsampled} of shuffle (non sub-sampled) mechanisms; see Lemma~\ref{lem:samplig}.
Then we reduce the computation of the ternary $|\chi|^{\alpha}$-DP of shuffle
  mechanisms for a {\em generic} triple of
  neighboring datasets to those that have a special
  structure (see Theorem \ref{thm:reduce_special_case}) --  this reduction step is one of the core technical results of this paper.
  Then we bound the ternary $|\chi|^{\alpha}$-DP of the shuffle
  mechanisms for triples of neighboring datasets having special structures by bounding the Pearson-Vajda divergence~\cite{wang2019subsampled} using some concentration properties (see Theorem~\ref{thm:ternary_special_case}).

\item Using the core technical result in Theorem
  \ref{thm:general_case}, we analyze privacy-convergence trade-offs of the CLDP-SGD algorithm (see Algorithm~\ref{algo:optimization-algo}) 
 for Lipschitz
  convex functions in Theorem \ref{thm:main-opt-result}. This partially
  resolves an open question posed in \cite{girgis2021shuffled-aistats},
  to extend their privacy analysis to RDP and significantly
  strengthening their privacy guarantees.
\item Numerically, we save a factor of $14\times$ in privacy ($\epsilon$)
  over the best known results for approximate DP for shuffling~\cite{feldman2020hiding}
  combined with strong composition \cite{kairouz2015composition} for
  $T=10^{5}, \gamma=0.001, n=10^{6}$. Translating these to privacy-performance
  operating point in distributed optimization, over the MNIST data set
  with $\ell_{\infty}$-norm clipping we numerically show gains: For the
  same privacy budget of $\epsilon=1.4$, we get a test performance of
  $80\%$ whereas using strong composition the test performance of \cite{feldman2020hiding}  is
  $70\%$; furthermore, we achieve $90\%$ accuracy with the total privacy budget $\epsilon= 2.91$, whereas, \cite{feldman2020hiding} (with strong composition) achieve the same accuracy with a total privacy budget of $\epsilon= 4.82$. See Section \ref{sec:numerics} for more results. 
\end{itemize}

\paragraph{Related work:}
We give a more complete literature review in Appendix~\ref{app_sec:LongerLitReview}, and focus here on the
works that are closest to the results presented in this paper.
\begin{itemize}[leftmargin=*]
\item \textit{Private optimization in the shuffled model:}
Recently,~\cite{ESA} and~\cite{girgis2021shuffled-aistats} have
proposed differentially private SGD algorithms for federated learning,
where at each iteration, each client applies an LDP mechanism on the
gradients with the existence of a secure shuffler between the clients
and the central server. However, the privacy analyses in these works
developed approximate DP using advanced composition theorems for DP
(\emph{e.g.,} \cite{dwork2010boosting,kairouz2015composition}), which
are known to be loose for composition \cite{abadi2016deep}. To the best of our knowledge, analyzing the private optimization framework using RDP and subsampling in the shuffled model is new to this paper.

\item \textit{Subsampled RDP:} The
works~\cite{mironov2019r,wang2019subsampled,zhu2019poission} have
studied the RDP of subsampled mechanisms \emph{without
  shuffling}. They demonstrated that this provides a tighter bound on
the total privacy loss than the bound that can be obtained using the
standard strong composition theorems. The RDP of the
shuffled model was very recently studied in \cite{girgis2021renyi}, but
without incorporating subsampling, which poses new technical challenges, as directly bounding the RDP of subsampled shuffle mechanisms is non-trivial. We overcome this by reducing our problem of computing RDP to bounding the ternary $|\chi|^{\alpha}$-DP, and bounding the latter is a core technical contribution of our paper.
\end{itemize}

\paragraph{Paper organization:} 
We give preliminaries and problem formulation in Section~\ref{sec:prelims_prob-form},
main results (upper and lower bounds, and privacy-convergence tradeoff) in Section~\ref{sec:main_results},
numerical results in Section~\ref{sec:numerics},
proof of the upper bound in Section~\ref{sec:general_case}, and
proof of the ternary DP of the shuffle model in Section~\ref{sec:proof_ternary_DP_shuffle}.
%Omitted details/proofs from this paper are given in the supplementary material.

%% file: preliminaries.tex
\section{Preliminaries and Problem Formulation}\label{sec:prelims_prob-form}
We use several privacy definitions throughout this paper. Among these, the local and central differential privacy definitions are standard. The other privacy definitions (Renyi DP and ternary $|\chi|^{\alpha}$-DP) are relatively less standard and we define them below.

\begin{defn}[Local Differential Privacy - LDP~\cite{kasiviswanathan2011can}]~\label{defn:LDPdef}
For $\epsilon_0\geq0$, a randomized mechanism $\calR:\calX\to\calY$ is said to be $\eps_0$-local differentially private (in short, $\eps_{0}$-LDP), if for every pair of inputs $d,d'\in\calX$, we have 
\begin{equation}~\label{ldp-def}
\Pr[\calR(d)\in \calS] \leq e^{\eps_0}\Pr[\calR(d')\in \calS], \qquad \forall \calS\subseteq\calY.
\end{equation}
\end{defn}

Let $\calD=\lbrace d_1,\ldots,d_n\rbrace$ denote a dataset comprising $n$ points from $\calX$. We say that two datasets $\calD=\lbrace d_1,\ldots,d_n\rbrace$ and $\calD^{\prime}=\lbrace d_1^{\prime},\ldots,d_n^{\prime}\rbrace$ are neighboring (and denoted by $\calD\sim\calD'$) if they differ in one data point, i.e., there exists an $i\in[n]$ such that $d_i\neq d'_i$ and for every $j\in[n],j\neq i$, we have $d_j=d'_j$.
\begin{defn}[Central Differential Privacy - DP \cite{Calibrating_DP06,dwork2014algorithmic}]\label{defn:central-DP}
For $\epsilon,\delta\geq0$, a randomized mechanism $\calM:\calX^n\to\calY$ is said to be $(\epsilon,\delta)$-differentially private (in short, $(\epsilon,\delta)$-DP), if for all neighboring datasets $\calD,\calD^{\prime}\in\calX^{n}$ and every subset $\calS\subseteq \calY$, we have
\begin{equation}~\label{dp_def}
\Pr\left[\calM(\calD)\in\calS\right]\leq e^{\eps_0}\Pr\left[\calM(\calD^{\prime})\in\calS\right]+\delta.
\end{equation}
\end{defn}

\begin{defn}[$(\lambda,\epsilon)$-RDP (Renyi Differential Privacy)~\cite{mironov2017renyi}]\label{defn:RDP}
A randomized mechanism $\calM:\calX^n\to\calY$ is said to have $\epsilon$-Renyi differential privacy of order $\lambda\in(1,\infty)$ (in short, $(\lambda,\epsilon(\lambda))$-RDP), if for any neighboring datasets $\calD$, $\calD'\in\calX^n$, the Renyi divergence of order $\lambda$ between $\calM(\calD)$ and $\calM(\calD')$ is upper-bounded by $\eps(\lambda)$, {\em i.e.},
\begin{equation}
D_{\lambda}(\calM(\calD)||\calM(\calD'))=\frac{1}{\lambda-1}\log\left(\mathbb{E}_{\theta\sim\calM(\calD')}\left[\left(\frac{\calM(\calD)(\theta)}{\calM(\calD')(\theta)}\right)^{\lambda}\right]\right)\leq \epsilon(\lambda),
\end{equation}
where $\calM(\calD)(\theta)$ denotes the probability that $\calM$ on input $\calD$ generates the output $\theta$.
\end{defn}

\begin{defn}[$\zeta$-Ternary $|\chi|^{\alpha}$-differential privacy~\cite{wang2019subsampled}]\label{defn:TDP}

A randomized mechanism $\calM:\calX^n\to\calY$ is said to have $\zeta$-ternary-$|\chi|^{\alpha}$-DP, if for any triple of mutually adjacent datasets $\calD,\calD',\calD''\in\calX^n$ ({\em i.e.}, they mutually differ in the same location), the ternary-$|\chi|^{\alpha}$ divergence of $\calM(\calD),\calM(\calD'),\calM(\calD')$ is upper-bounded by $(\zeta(\alpha))^{\alpha}$ for all $\alpha\geq 1$ (where $\zeta$ is a function from $\bbR^+$ to $\bbR^+$), {\em i.e.},
\begin{equation*}
%\begin{aligned}
D_{|\chi|^{\alpha}}\left(\calM(\calD),\calM(\calD')||\calM(\calD'')\right) := %\right]^{1/\alpha}&=\sup_{\calD\sim\calD'\sim\calD^{''}}\left[
\mathbb{E}_{\calM(\calD'')}\left[\left|\frac{\calM(\calD)-\calM(\calD')}{\calM(\calD'')}\right|^{\alpha}\right]
\leq\(\zeta(\alpha)\)^{\alpha}.
%\end{aligned}
\end{equation*}
\end{defn}
The ternary $|\chi|^{\alpha}$-DP was proposed in~\cite{wang2019subsampled} to characterize the RDP of the sub-sampled mechanism without shuffling. In this work, we analyze the ternary  $|\chi|^{\alpha}$-DP of the shuffled mechanism to bound the RDP of the sub-sampled shuffle model. 
 
We can use the following result for converting the RDP guarantees of a mechanism to its central DP guarantees. To the best of our knowledge, this result gives the best conversion.

\begin{lemma}[From RDP to DP~\cite{canonne2020discrete,Borja_HypTest-RDP20}]\label{lem:RDP_DP} 
Suppose for any $\lambda>1$, a mechanism $\calM$ is $\left(\lambda,\epsilon\left(\lambda\right)\right)$-RDP. Then, the mechanism $\calM$ is $\left(\epsilon,\delta\right)$-DP, where $\delta>0$ is arbitrary and $\eps$ is given by % $\epsilon,\delta$ are define below: 
\begin{equation*}
\epsilon = \min_{\lambda} \(\epsilon\left(\lambda\right)+\frac{\log\left(1/\delta\right)+\left(\lambda-1\right)\log\left(1-1/\lambda\right)-\log\left(\lambda\right)}{\lambda-1}\).
\end{equation*}
\end{lemma}

\subsection{Problem formulation} We consider a distributed private learning setup comprising a set of $n$ clients, where the $i$th client has a data point $d_i$ drawn from a universe $\calX$ for $i\in\left[n\right]$; see also Figure~\ref{fig:Setup}. Let $\calD=\left(d_1,\ldots,d_n\right)$ denote the entire training dataset. The clients are connected to an untrusted server in order to solve the following empirical risk minimization (ERM) problem
\begin{equation}\label{eq:problem-formulation}
\min_{\theta\in\mathcal{C}} \Big( F(\theta,\calD) := \frac{1}{n}\sum_{i=1}^{n} f(\theta,d_i) \Big),
\end{equation}
where $\mathcal{C}\subset \mathbb{R}^d$ is a closed convex set, and $f:\calC\times\calD\to\bbR$ is the loss function. Our goal is to construct a global learning model $\theta$ via stochastic gradient descent (SGD) while preserving privacy of individual data points in the training dataset $\calD$ by providing strong DP guarantees.

\begin{algorithm}[t]
\caption{$\mathcal{A}_{\text{cldp}}$: CLDP-SGD} \label{algo:optimization-algo}
\begin{algorithmic}[1]
\State {\bf Input:} Datasets $\mathcal{D}=\left(d_1,\ldots,d_n\right)$, LDP privacy parameter $\epsilon_0$, gradient norm bound $C$, and learning rate schedule $\{\eta_t\}$. 
\State \textbf{Initialize:} $\theta_0\in\mathcal{C}$ 
\For {$t\in\left[T\right]$}
\State \textbf{Client sampling:} A random set $\mathcal{U}_t$ of $k$ clients is chosen.
\For  {clients $i\in\mathcal{U}_t$}
\State \textit{Compute gradient:} $\mathbf{g}_t\left(d_{i}\right)\gets \nabla_{\theta_t}f\left(\theta_t,d_{i}\right)$
\State \textit{Clip gradient:} $\tilde{\mathbf{g}}_t\left(d_{i}\right)\gets \mathbf{g}_t\left(d_{i}\right)/\max\left\{1,\frac{\|\mathbf{g}_t\left(d_{i}\right)\|_p}{C}\right\}$
\State Client $i$ sends $\mathcal{R}_p\left(\tilde{\mathbf{g}}_t\left(d_{i}\right)\right)$ to the shuffler.
\EndFor
\State \textbf{Shuffling:} The shuffler sends random permutation of $\{\mathcal{R}_p\left(\tilde{\mathbf{g}}_t\left(d_{i}\right)\right) : i\in\calU_t\}$ to the server.
\State \textbf{Aggregate:} $\overline{\mathbf{g}}_t\gets \frac{1}{k}\sum_{i\in\mathcal{U}_t}\mathcal{R}_p\left(\tilde{\mathbf{g}}_t\left(d_{i}\right)\right)$
\State \textbf{Descent Step:} $\theta_{t+1}\gets \prod_{\mathcal{C}}\left( \theta_t -\eta_t \overline{\mathbf{g}}_t\right)$, where $\prod_{\calC}$ is the projection operator onto the set $\calC$. \\
\EndFor
\State \textbf{Output:} The model $\theta_{T}$ and the privacy parameters $\epsilon$, $\delta$.
\end{algorithmic}
\end{algorithm}

We revisit the CLDP-SGD algorithm presented in~\cite{girgis2021shuffled-aistats} and described in Algorithm~\ref{algo:optimization-algo} to solve the ERM~\eqref{eq:problem-formulation}. In each step of CLDP-SGD, we choose uniformly at random a set $\mathcal{U}_t$ of $k\leq n$ clients out of $n$ clients. Each client $i\in\mathcal{U}_t$ computes and clips the $\ell_p$-norm of the gradient $\nabla_{\theta_t} f\left(\theta_t,d_{i}\right)$ to apply the LDP mechanism $\mathcal{R}_p$, where $\calR_p:\calB_{p}^{d}\to \lbrace 0,1\rbrace^{b}$ is an $\epsilon_0$-LDP mechanism when inputs come from the $\ell_p$-norm ball $\calB_p^d=\{\bx\in\bbR^d:\|\bx\|_p\leq1\}$, where $\|\bx\|_p=\(\sum_{i=1}^d |x_i|^p\)^{1/p}$ denotes the $\ell_p$-norm of the vector $\bx\in\bbR^d$. In~\cite{girgis2021shuffled-aistats}, we proposed different $\epsilon_0$-LDP mechanisms for general $\ell_p$-norm balls. After that, the shuffler randomly permutes the received $k$ gradients $\lbrace \mathcal{R}_p\left(\tilde{\mathbf{g}}_t\left(d_{i}\right)\right)\rbrace_{i\in\mathcal{U}_{t}}$ and sends them to the server. Finally, the server takes the average of the received gradients and updates the parameter vector.
Our main contribution in this work is to present a stronger privacy analysis of the CLDP-SGD algorithm by characterizing the RDP of the sub-sampled shuffle model.

%% file: main_results.tex
\section{Main Results}\label{sec:main_results}

In this section, we present our main results.  First, we characterize the RDP of the subsampled shuffle mechanism by presenting an upper bound in Theorem~\ref{thm:general_case} and a lower bound in Theorem~\ref{thm:lower_bound}. We then present the privacy-convergence trade-offs of the CLDP-SGD Algorithm in Theorem~\ref{thm:main-opt-result}. 

Consider an arbitrary $\epsilon_0$-LDP mechanism $\calR$, whose range is a discrete set $\left[B\right]=\left\{1,\ldots,B\right\}$ for some $B\in\mathbb{N}:=\{1,2,3,\hdots\}$. Here, $[B]$ could be the whole of $\mathbb{N}$. Let $\calM(\calD)$ be a subsampled shuffle mechanism defined as follows: First subsample $k\leq n$ clients of the $n$ clients (without replacement), where $\gamma=\frac{k}{n}$ denotes the sampling parameter. Each client $i$ out of the $k$ selected clients applies $\calR$ on $d_i$ and sends $\calR(d_i)$ to the shuffler,\footnote{With a slight abuse of notation, in this paper we write $\calR(d_i)$ to denote that $\calR$ takes as its input the gradient computed on $d_i$ using the current parameter vector.} who randomly permutes the received $k$ inputs and outputs the result. To formalize this, let $\calH_k:\calY^k\to\calY^k$ denote the shuffling operation that takes $k$ inputs and outputs their uniformly random permutation. Let $\samp_{k}^{n}:\calX^{n}\to\calX^{k}$ denote the sampling operation for choosing a random subset of $k$ elements from a set of $n$ elements. We define the subsampled-shuffle mechanism as
\begin{equation}\label{shuffle-mech}
\calM\left(\calD\right) := \calH_{k}\circ \samp_{k}^{n}\left(\calR\left(d_1\right),\ldots,\calR\left(d_n\right)\right).
\end{equation} 
Observe that each iteration of Algorithm~\ref{algo:optimization-algo} can be represented as an output of the subsampled shuffle mechanism $\calM$. Thus, to analyze the privacy of Algorithm~\ref{algo:optimization-algo}, it is sufficient to analyze the privacy of a sequence of identical $T$ subsampled shuffle mechanisms, and then apply composition theorems.
    
\paragraph{Histogram notation.} It will be useful to define the following notation. Since the output of $\calH_k$ is a random permutation of the $k$ outputs of $\calR$ (subsampling is not important here), the server cannot associate the $k$ messages to the clients; and the only information it can use from the messages is the histogram, i.e., the number of messages that give any particular output in $[B]$. 
We define a set $\calA_{B}^{k}$ as 
\begin{equation}\label{histogram-set}
\calA_B^{k}=\bigg\lbrace \bh=\left(h_1,\ldots,h_B\right):\sum_{j=1}^{B}h_j=k\bigg\rbrace,
\end{equation}
to denote the set of all possible histograms of the output of the shuffler with $k$ inputs. Therefore, we can assume, without loss of generality (w.l.o.g.), that the output of $\calM$ is a distribution over $\calA_B^{k}$.

Our main results for the RDP of the subsampled shuffled mechanism (defined in \eqref{shuffle-mech}) are given below. Our first result provides an upper bound (stated in Theorem~\ref{thm:general_case} and proved in Section~\ref{sec:general_case}) and the second result provides a lower bound (stated in Theorem~\ref{thm:lower_bound} and proved in Section~\ref{app:lower-bound-proof}).
\begin{theorem}[Upper Bound]\label{thm:general_case} 
For any $n\in\mathbb{N}$, $k\leq n$, $\epsilon_0\geq 0$, and any integer $\lambda\geq 2$, the RDP of the subsampled shuffle mechanism $\calM$ (defined in \eqref{shuffle-mech}) is upper-bounded by
\begin{equation*}%\label{eqn:1st_bound}
\epsilon(\lambda) \leq \frac{1}{\lambda-1}\log\left(1+4\binom{\lambda}{2}\gamma^{2}\frac{\left(e^{\epsilon_0}-1\right)^2}{\overline{k}e^{\epsilon_0}}+\sum_{j=3}^{\lambda}\binom{\lambda}{j}\gamma^{j} j \Gamma\left(j/2\right) \left(\frac{2\left(e^{2\epsilon_0}-1\right)^2}{\overline{k}e^{2\epsilon_0}}\right)^{j/2}+\Upsilon\right),
\end{equation*}
where $\overline{k}=\floor{\frac{k-1}{2e^{\epsilon_0}}}+1$, $\gamma=\frac{k}{n}$, and $\Gamma\left(z\right)=\int_{0}^{\infty}x^{z-1}e^{-x}dx$ is the Gamma function. The term $\Upsilon$ is given by $\Upsilon=\left(\left(1+\gamma \frac{e^{2\epsilon_0}-1}{e^{\epsilon_0}}\right)^{\lambda}-1-\lambda\gamma \frac{e^{2\epsilon_0}-1}{e^{\epsilon_0}}\right)e^{-\frac{k-1}{8e^{\epsilon_0}}}$.
\end{theorem}
%We prove Theorem~\ref{thm:general_case} in Section~\ref{sec:general_case}.
%
\begin{theorem}[Lower Bound]\label{thm:lower_bound}
For any $n\in\mathbb{N}$, $k\leq n$, $\epsilon_0\geq 0$, and any integer $\lambda\geq 2$, the RDP of the subsampled shuffle mechanism $\calM$ (defined in \eqref{shuffle-mech}) is lower-bounded by
\begin{equation*}%\label{eq:lower-bound}
\epsilon\left(\lambda\right) 
\geq \frac{1}{\lambda-1}\log\left(1+\binom{\lambda}{2}\gamma^2\frac{\left(e^{\epsilon_0}-1\right)^2}{ke^{\epsilon_0}}+\sum_{j=3}^{\lambda} \binom{\lambda}{j} \gamma^j \left(\frac{\left(e^{2\epsilon_0}-1\right)}{k e^{\epsilon_0}}\right)^{j}\mathbb{E}\left(m-\frac{k}{e^{\epsilon_0}+1}\right)^{j}\right), 
\end{equation*}
where expectation is taken w.r.t.\ the binomial r.v.\ $m\sim\text{Bin}\left(k,p\right)$ with parameter $p=\frac{1}{e^{\epsilon_0}+1}$. 
\end{theorem}

The CLDP-SGD algorithm and its privacy-convergence trade-offs (stated in Theorem~\ref{thm:main-opt-result} below) are given for a general local randomizer $\calR_p$ (whose inputs comes from an $\ell_p$-ball for any $p\in[1,\infty]$) that satisfies the following conditions:
{\sf (i)} The randomized mechanism $\calR_p$ is an $\epsilon_0$-LDP mechanism. 
{\sf (ii)} The randomized mechanism $\calR_p$ is unbiased, i.e., $\mathbb{E}\left[\calR_p\left(\mathbf{x}\right)|\mathbf{x}\right]=\mathbf{x}$ for all $\mathbf{x}\in\calB_{p}(a)$, where $a$ is the radius of the ball $\calB_{p}$.
{\sf (iii)} The output of the randomized mechanism $\calR_p$ can be represented using $B\in\mathbb{N}^{+}$ bits. 
{\sf (iv)} The randomized $\calR_p$ has a bounded variance: $\sup_{\mathbf{x}\in\calB_{p}(a)}\mathbb{E}\|\calR_p\left(\mathbf{x}\right)-\mathbf{x}\|_{2}^{2} \leq G_{p}^2(a)$, where $G_{p}^2$ is a function from $\mathbb{R}^{+}$ to $\mathbb{R}^{+}$.

\cite{girgis2021shuffled-aistats} proposed unbiased $\epsilon_0$-LDP mechanisms $\calR_p$ for several values of norms $p\in[1,\infty]$ that require $b=\mathcal{O}\left(\log\left(d\right)\right)$ bits of communication and satisfy the above conditions. In this paper, achieving communication efficiency is not our goal (though we also achieve that since the $\eps_0$-LDP mechanism $\calR_p$ that we use takes values in a discrete set), as our main focus is on analyzing the RDP of the subsampled shuffle mechanism. If we use the $\eps_0$-LDP mechanism $\calR_p$ from \cite{girgis2021shuffled-aistats}, we would also get similar gains in communication as were obtained in \cite{girgis2021shuffled-aistats}. %; however, since our focus is different, we prefer to omit the discussion on communication complexity of our mechanism.

The privacy-convergence trade-off of our algorithm $\calA_{\text{cldp}}$ is given below.
\begin{theorem}[Privacy-Convergence tradeoffs]\label{thm:main-opt-result} 
Let the set $\mathcal{C}$ be convex with diameter $D$ and the function $f\left(\theta;.\right):\mathcal{C}\times\calD\to \mathbb{R}$ be convex and $L$-Lipschitz continuous with respect to the $\ell_g$-norm, which is the dual of the $\ell_p$-norm. Let $\theta^{*}=\arg\min_{\theta\in\mathcal{C}} F\left(\theta\right)$ denote the minimizer of the problem~\eqref{eq:problem-formulation}. For $\gamma=\frac{k}{n}$, if we run Algorithm $\calA_{\emph{cldp}}$ over $T$ iterations, then we have 
\begin{enumerate}
\item \textbf{Privacy:} $\calA_{\emph{cldp}}$ is $\left(\epsilon,\delta\right)$-DP, where $\delta>0$ is arbitrary and $\eps$ is given by 
\begin{equation}\label{final-epsilon}
\epsilon=\min_{\lambda} \(T\epsilon\left(\lambda\right)+ \frac{\log\left(1/\delta\right)+\left(\lambda-1\right)\log\left(1-1/\lambda\right)-\log\left(\lambda\right)}{\lambda-1}\),
\end{equation}
where $\epsilon\left(\lambda\right)$ is the RDP of the subsampled shuffle mechanism given in Theorem~\ref{thm:general_case}. 

\item \textbf{Convergence:} If we run $\calA_{\emph{cldp}}$ with $\eta_t=\frac{D}{G\sqrt{t}}$, where $G^2=\max\lbrace d^{1-\frac{2}{p}},1\rbrace L^2+\frac{G_{p}^2(L)}{\gamma n}$, we get \begin{align*}
\mathbb{E}\left[F\left(\theta_{T}\right)\right] - F\left(\theta^{*}\right) \leq \calO\(\frac{DG\log(T)}{\sqrt{T}}\). %\label{general-convergence}
\end{align*}
\end{enumerate}
\end{theorem}

The proof outline of Theorem~\ref{thm:main-opt-result} is as follows: Note that $\calA_{\text{cldp}}$ is an iterative algorithm, where in each iteration we use the subsampled shuffle mechanism as defined in \eqref{shuffle-mech}, for which we have computed the RDP guarantees in Theorem~\ref{thm:general_case}. Now, for the privacy analysis of $\calA_{\text{cldp}}$, we use the adaptive composition theorem from {\cite[Proposition~1]{mironov2017renyi}} and then use the RDP to DP conversion given in Lemma~\ref{lem:RDP_DP}. For the convergence analysis, we use a standard non-private SGD convergence result and compute the required parameters for that. See Section~\ref{app_sec:OptPerf} for a complete proof of Theorem~\ref{thm:main-opt-result}. 

\begin{remark}  
Note that our convergence bound is affected by the variance of the $\epsilon_0$-LDP mechanism $\calR_p$. 
For example, when $f$ is $L$-Lipschitz continuous w.r.t.\ the $\ell_2$-norm, we can use the LDP mechanism $\calR_2$ proposed in~\cite{bhowmick2018protection} that has variance $G_2^2(L)=14 L^2 d\big(\frac{e^{\epsilon_0}+1}{e^{\epsilon_0}-1}\big)^2$; and when $f$ is $L$-Lipschitz continuous w.r.t.\ the $\ell_{1}$-norm or $\ell_{\infty}$-norm, we can use the LDP mechanisms $\calR_{\infty}$ or $\calR_1$, respectively, proposed in~\cite{girgis2021shuffled-aistats} that have variances $G_{\infty}^2(L)=L^2 d^2\big(\frac{e^{\epsilon_0}+1}{e^{\epsilon_0}-1}\big)^2$ and $G_{1}^2(L)=L^2 d\big(\frac{e^{\epsilon_0}+1}{e^{\epsilon_0}-1}\big)^2$, respectively. By plugging these variances $G_{p}^2(L)$ (for $p=1,2,\infty$) into Theorem~\ref{thm:main-opt-result}, we get the convergence rate of the $L$-Lipschitz continuous loss function w.r.t.\ the $\ell_p$-norm (for $p=\infty,2,1$).
\end{remark}
\begin{remark} The privacy parameter in \eqref{final-epsilon} is not in a closed form expression and could be obtained by solving an optimization problem. However, we numerically compute it for several interesting regimes of parameters in our numerical experiments; see Section~\ref{sec:numerics} for more details.
\end{remark}

%% file: numerics.tex
\section{Numerical Results}\label{sec:numerics}

In this section, we present numerical experiments to show the performance of our bounds on RDP of the subsampled shuffle mechanism and its usage for getting approximate DP of Algorithm~\ref{algo:optimization-algo} for training machine learning models.

\begin{figure}[t]
\centering
\begin{subfigure}{0.31\textwidth}
    \centering 
  \includegraphics[scale=0.38]{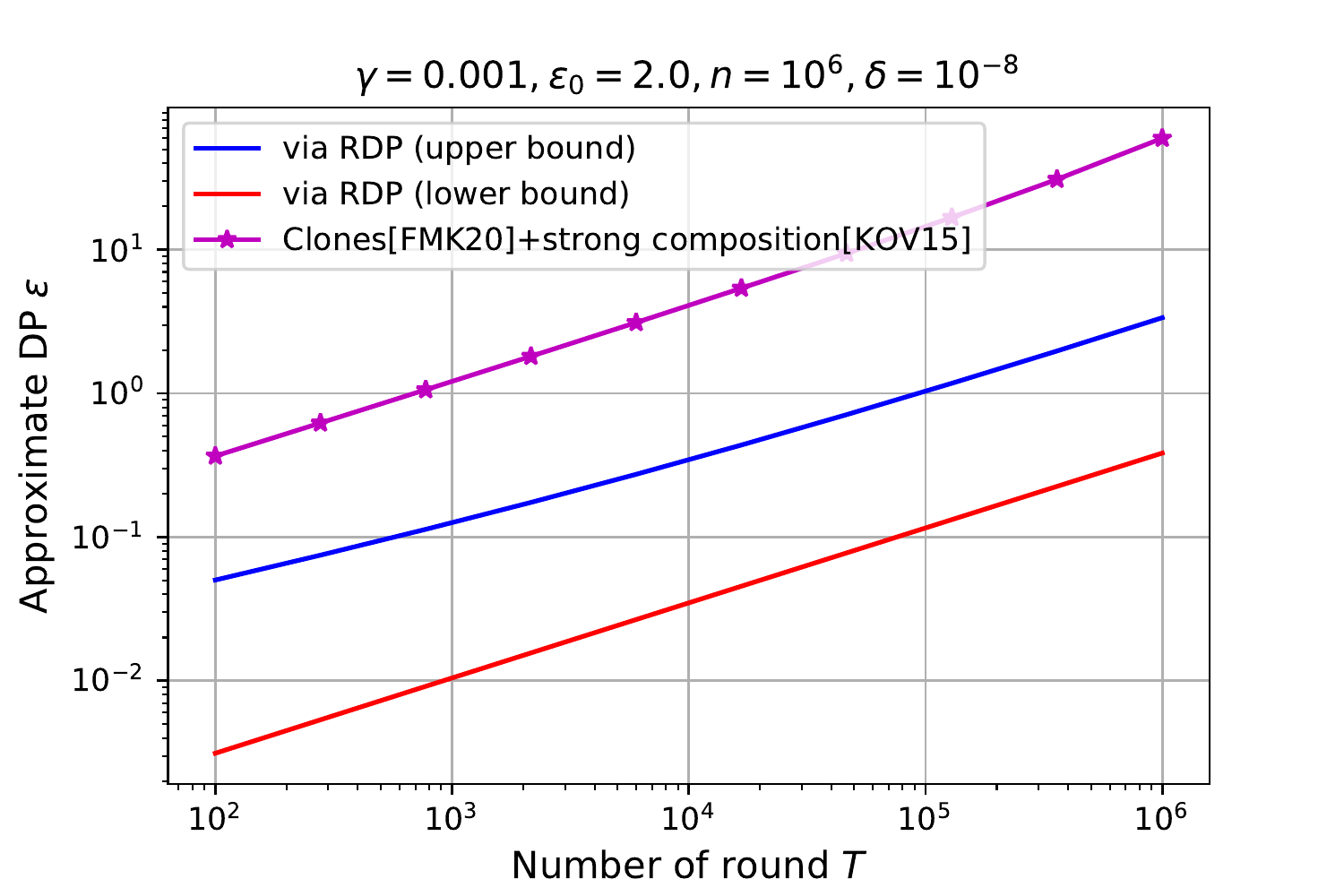}
  \caption{Approx.\ DP as a function of $T$ for $\epsilon_0=2$, $\gamma=0.001$, $n=10^{6}$}
  \label{fig:Dp_com_1}
\end{subfigure}\hfil 
\begin{subfigure}{0.31\textwidth}
    \centering 
  \includegraphics[scale=0.38]{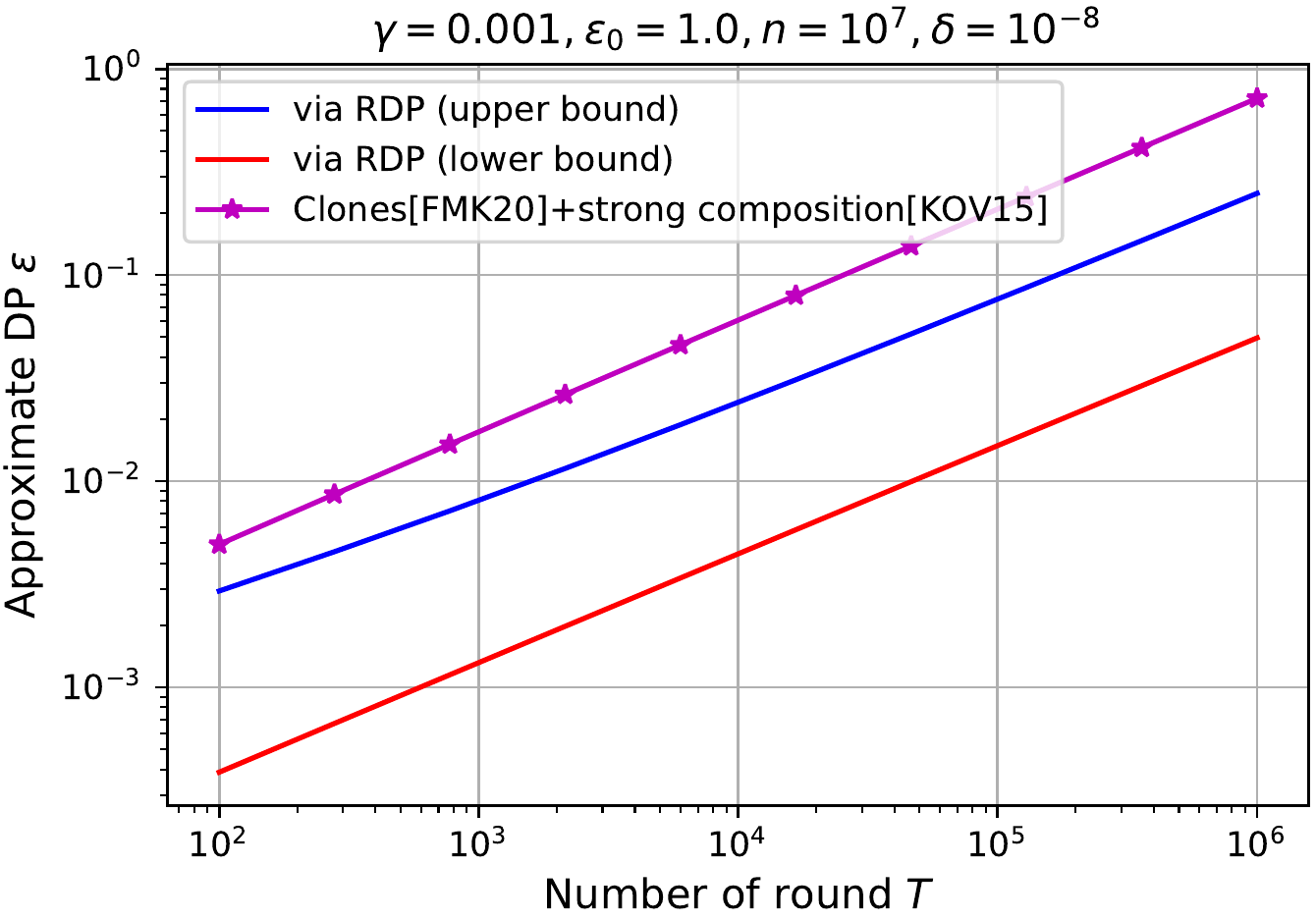}
  \caption{Approx.\ DP as a function of $T$ for $\epsilon_0=1$, $\gamma=0.001$, $n=10^{7}$}
  \label{fig:Dp_com_2}
\end{subfigure}\hfil 
\begin{subfigure}{0.31\textwidth}
    \centering 
  \includegraphics[scale=0.38]{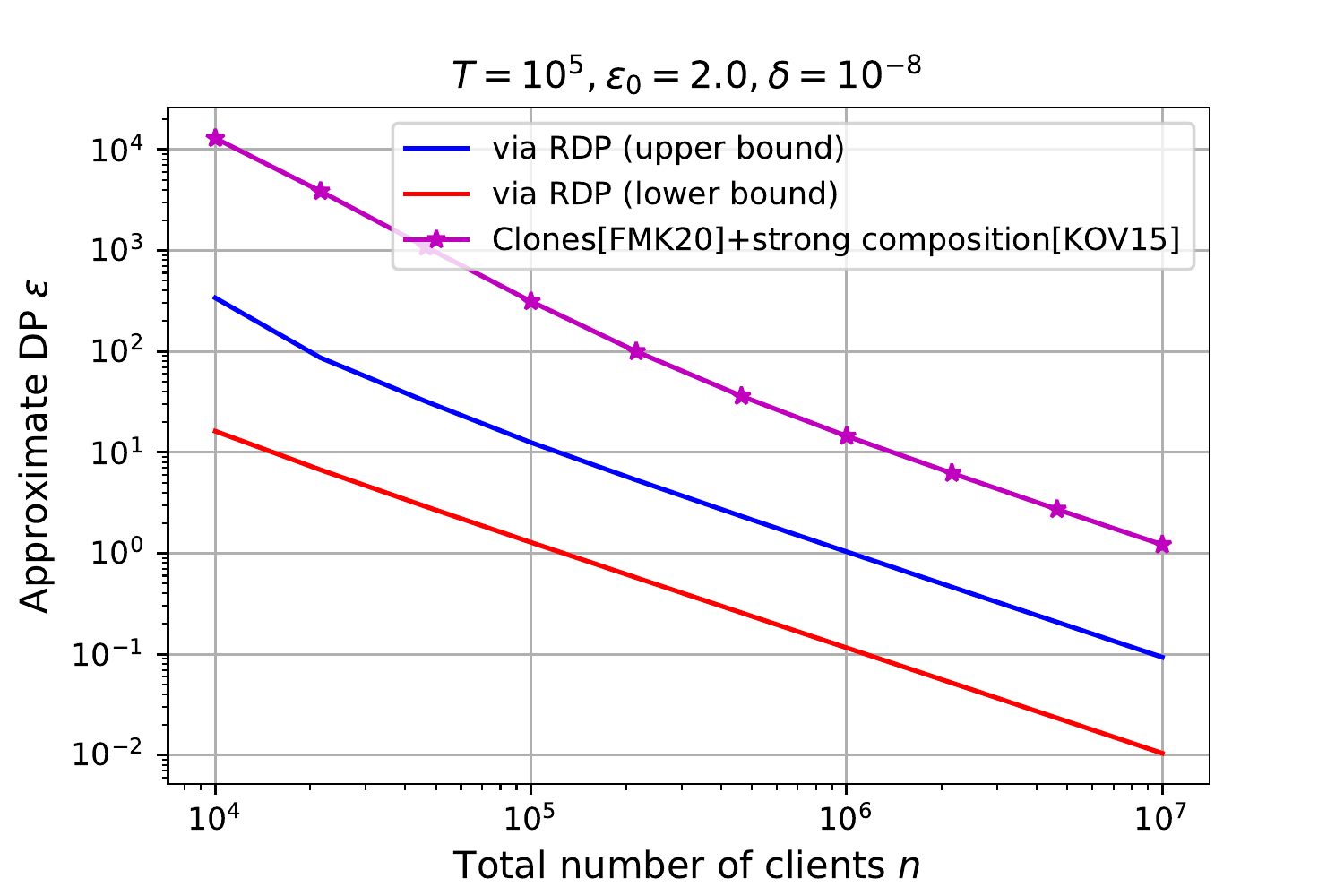}
  \caption{Approx.\ DP as a function of $n$ for $\epsilon_0=2$, $\gamma n = 10^3$, $T=10^{5}$}
  \label{fig:Dp_com_3}
\end{subfigure}
\caption{Comparison of several bounds on the Approximate $\left(\epsilon,\delta\right)$-DP for composition of a sequence of subsampled shuffle mechanisms for $\delta=10^{-8}$: {\sf (i)} Approximate DP obtained from our upper bound on the RDP in Theorem~\ref{thm:general_case} (blue); {\sf (ii)} Approximate DP obtained from our lower bound on the RDP in Theorem~\ref{thm:lower_bound} (red); and {\sf (iii)} Applying the strong composition theorem~\cite{kairouz2015composition} after getting the Approximate DP of the shuffled model given in~\cite{feldman2020hiding} with subsampling~\cite{Jonathan2017sampling} (magenta).}
\label{fig:DP_composition}
\end{figure}

\begin{figure}[t]
\centering
\begin{subfigure}{0.31\textwidth}
    \centering 
  \includegraphics[scale=0.38]{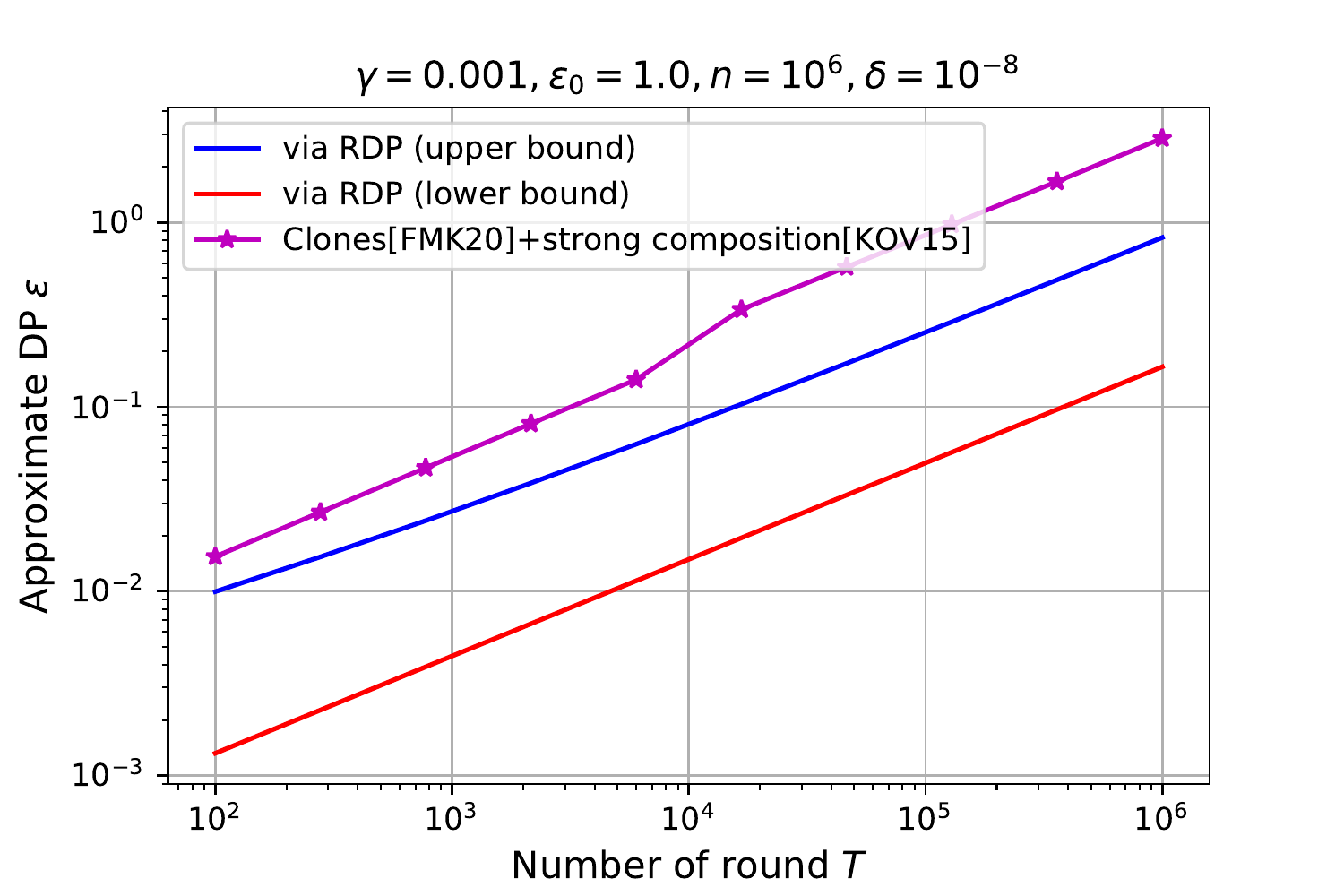}
  \caption{Approx.\ DP as a function of $T$ for $\epsilon_0=1$, $\gamma=0.001$, $n=10^{6}$}
  \label{fig:Dp_1_1}
\end{subfigure}\hfil 
\begin{subfigure}{0.31\textwidth}
    \centering 
  \includegraphics[scale=0.38]{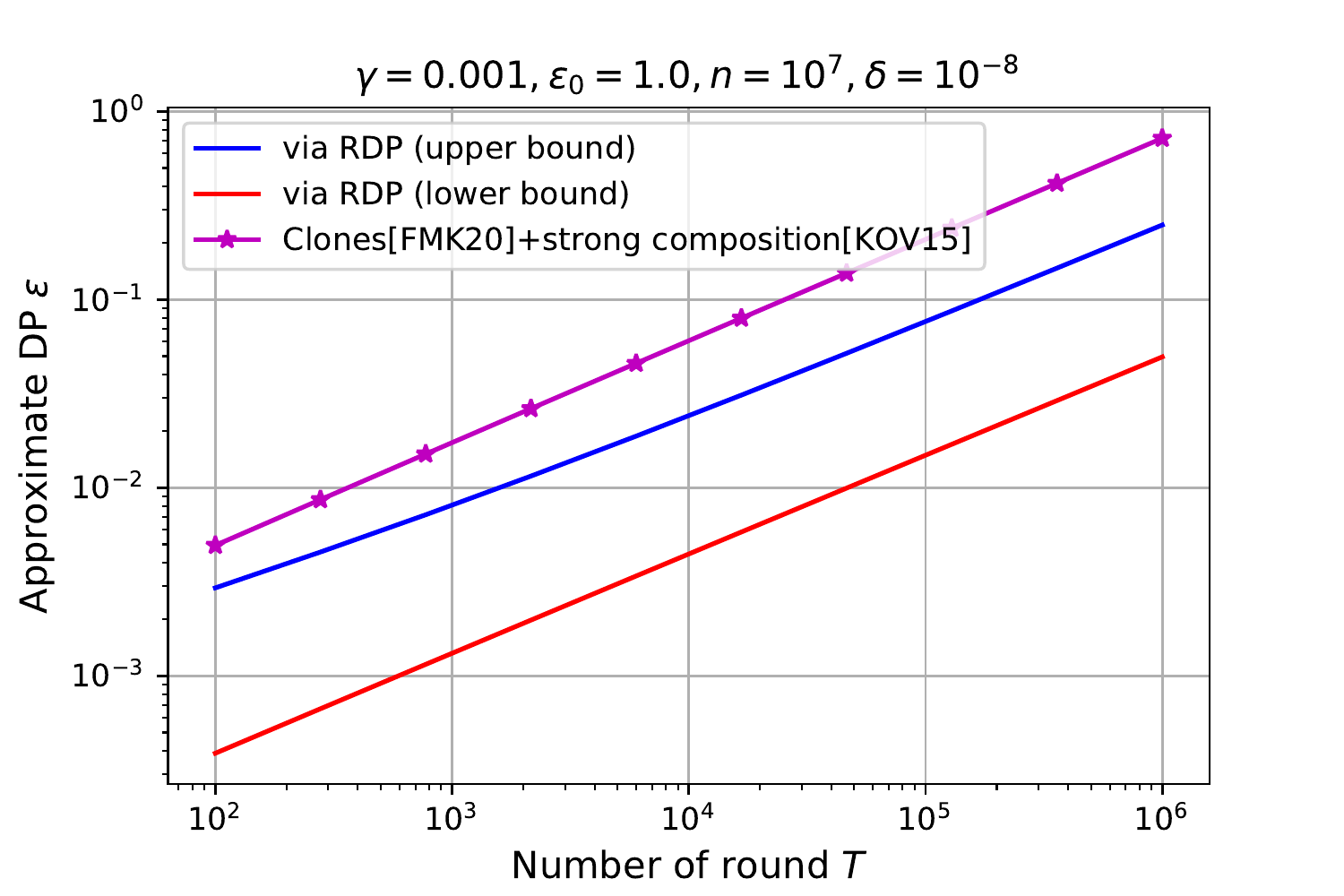}
  \caption{Approx.\ DP as a function of $T$ for $\epsilon_0=1$, $\gamma=0.001$, $n=10^{7}$}
  \label{fig:Dp_2_1}
\end{subfigure}\hfil 
\begin{subfigure}{0.3\textwidth}
    \centering 
  \includegraphics[scale=0.38]{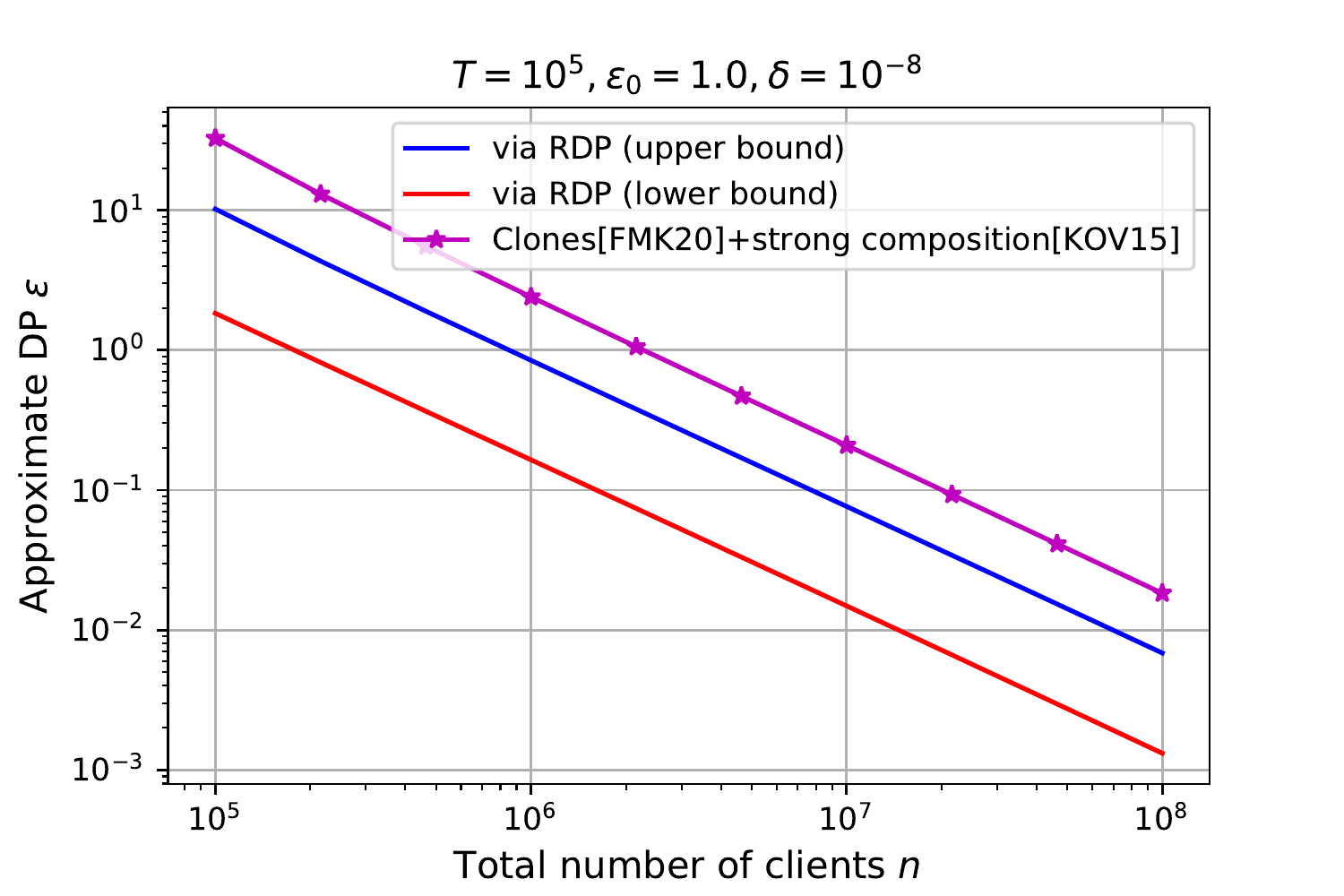}
  \caption{Approx.\ DP as a function of $n$ for $\epsilon_0=1$, $\gamma n = 10^4$, $T=10^{5}$}
  \label{fig:Dp_3_1}
\end{subfigure}
\caption{Comparison of our bound on the Approximate $\left(\epsilon,\delta\right)$-DP (blue) for composition of a sequence of subsampled shuffle mechanisms for $\delta=10^{-8}$ with applying the strong composition theorem~\cite{kairouz2015composition} after getting the Approximate DP of the shuffled model given in~\cite{feldman2020hiding} with subsampling~\cite{Jonathan2017sampling} (magenta).}
\label{fig:DP_composition_compare}
\end{figure}

\begin{figure}[t]
\centering
\begin{subfigure}{0.31\textwidth}
    \centering 
  \includegraphics[scale=0.38]{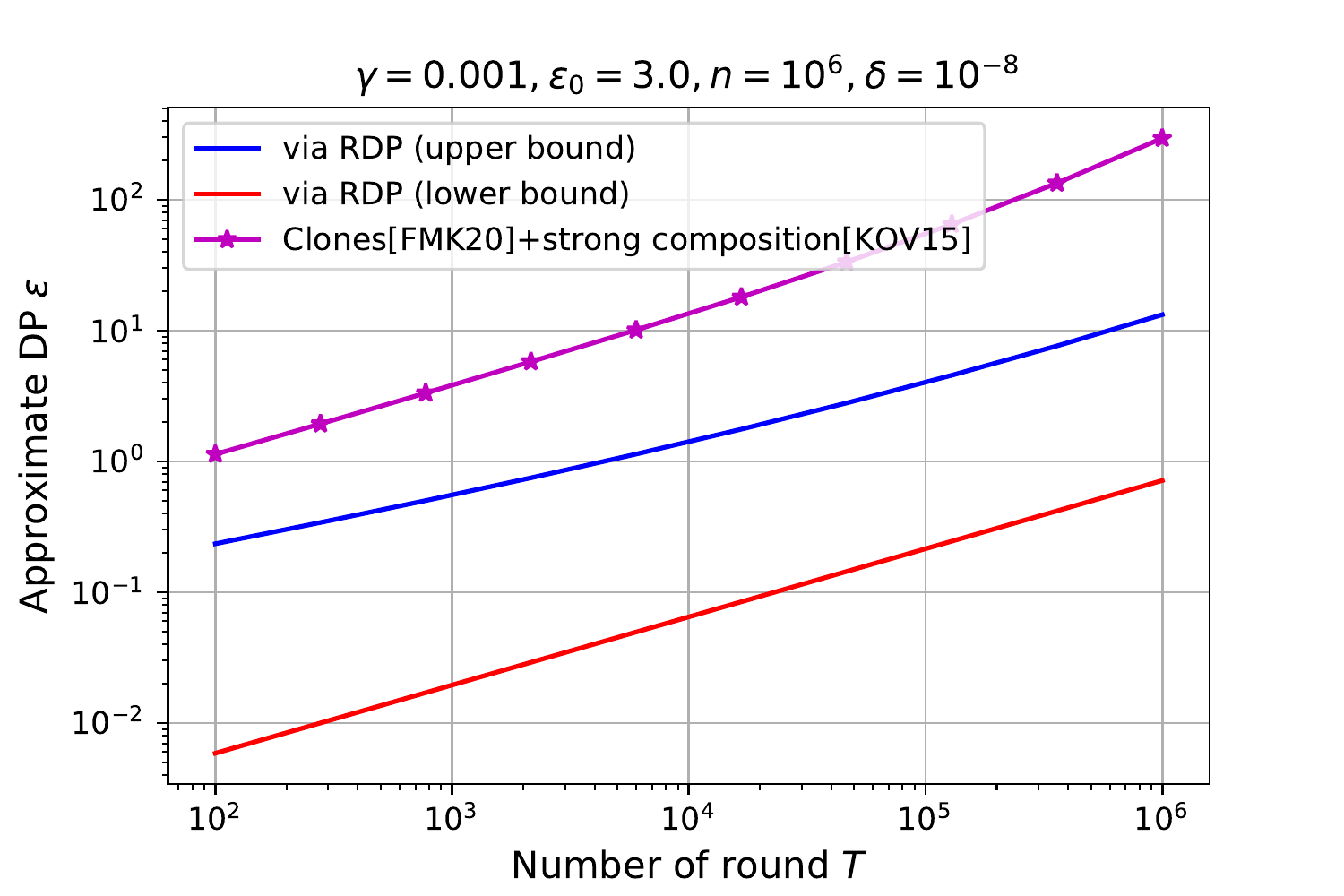}
  \caption{Approx.\ DP as a function of $T$ for $\epsilon_0=3$, $\gamma=0.001$, $n=10^{6}$}
  \label{fig:Dp_com_1_1}
\end{subfigure}\hfil 
\begin{subfigure}{0.31\textwidth}
    \centering 
  \includegraphics[scale=0.38]{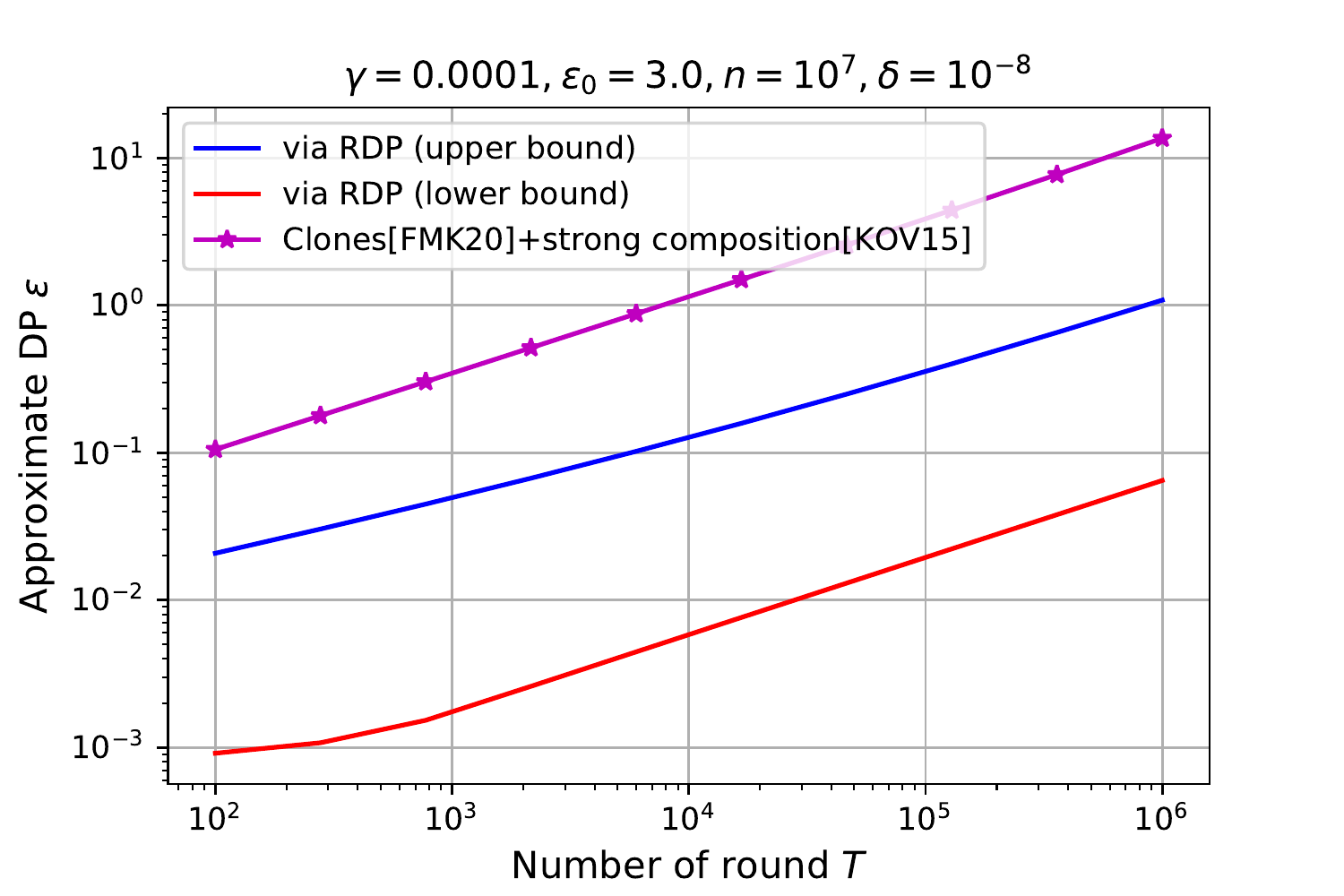}
  \caption{Approx.\ DP as a function of $T$ for $\epsilon_0=3$, $\gamma=0.0001$, $n=10^{7}$}
  \label{fig:Dp_com_2_1}
\end{subfigure}\hfil 
\begin{subfigure}{0.31\textwidth}
    \centering 
  \includegraphics[scale=0.38]{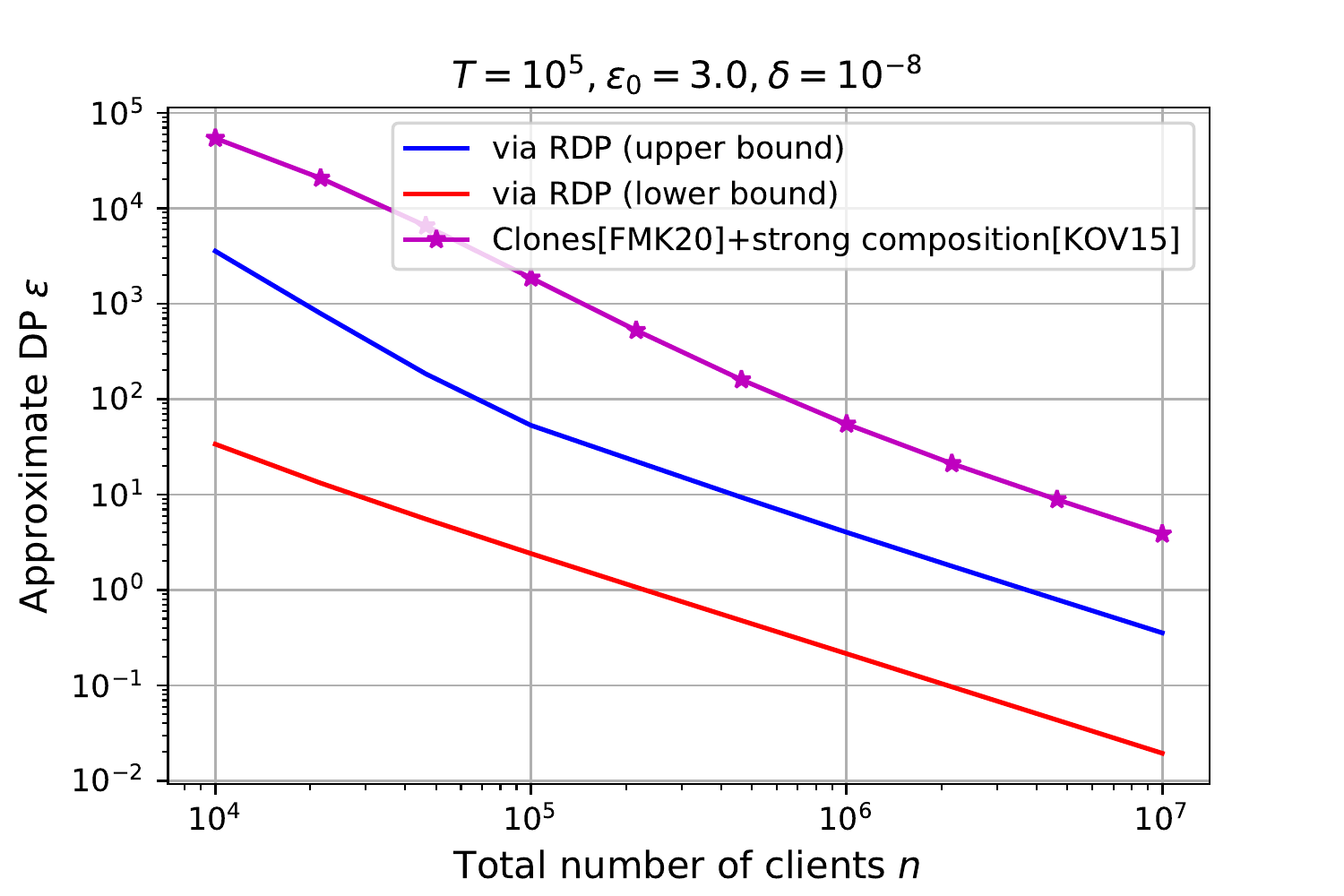}
  \caption{Approx.\ DP as a function of $n$ for $\epsilon_0=3$, $\gamma n = 10^3$, $T=10^{5}$}
  \label{fig:Dp_com_3_1}
\end{subfigure}
\caption{Comparison of several bounds on the Approximate $\left(\epsilon,\delta\right)$-DP for composition of a sequence of subsampled shuffle mechanisms for $\delta=10^{-8}$: {\sf (i)} Approximate DP obtained from our upper bound on the RDP in Theorem~\ref{thm:general_case} (blue); {\sf (ii)} Approximate DP obtained from our lower bound on the RDP in Theorem~\ref{thm:lower_bound} (red); and {\sf (iv)} Applying the strong composition theorem~\cite{kairouz2015composition} after getting the Approximate DP of the shuffled model given in~\cite{feldman2020hiding} with subsampling~\cite{Jonathan2017sampling} (magenta).}
\label{fig:DP_composition_additional}
\end{figure}

\paragraph{Composition of a sequence of subsampled shuffle models:}  
In Figures~\ref{fig:DP_composition} and \ref{fig:DP_composition_compare}, we plot several bounds on the approximate
$\left(\epsilon,\delta\right)$-DP for a composition of $T$ mechanisms $\left(\calM_1,\ldots,\calM_T\right)$, where $\calM_t$ is a subsampled shuffle mechanism for $t\in \left[T\right]$. In all our experiments reported in Figures~\ref{fig:DP_composition} and \ref{fig:DP_composition_compare}, we fix $\delta=10^{-8}$. In Figures~\ref{fig:Dp_com_3} and \ref{fig:Dp_3_1}, we fix the number of subsampled clients per iteration to be $k=\gamma n = 10^3$ and $k=10^4$, respectively. Hence, the subsampling parameter $\gamma$ varies with $n$. We observe that our new bound on the RDP of the subsampled shuffle mechanism achieves a significant saving in total privacy $\epsilon$ compared to the state-of-the-art. For example, we save a factor of $14\times$ compared to  the bound on DP~\cite{feldman2020hiding} with strong composition theorem~\cite{kairouz2015composition} in computing the overall privacy parameter $\epsilon$ for number of iterations $T=10^{5}$, subsampling parameter $\gamma=0.001$, LDP parameter $\epsilon_0=2$, and number of clients $n=10^{6}$.

Observe that our RDP bound presented in~\ref{thm:general_case} is general for any values of LDP parameter $\epsilon_0$, number of clients $n$, and RDP order $\lambda\geq 2$. On the other hand, the result of the privacy amplification by shuffling presented in~\cite{feldman2020hiding} is valid under the condition on the LDP parameter:
\begin{equation}\label{eqn:condition_clones}
\epsilon_0\leq \log\left(\frac{n}{16\log(2/\delta)}\right).
\end{equation} 
Furthermore, the result of the privacy amplification by shuffling presented in~\cite{balle2019privacy} is valid under the condition on the LDP parameter:
\begin{equation}\label{eqn:condition_blanket}
\epsilon_0\leq \frac{1}{2}\log\left(\frac{n}{\log(1/\delta)}\right).
\end{equation} 
Thus, if the conditions in~\eqref{eqn:condition_clones}-\eqref{eqn:condition_blanket} do not hold, then the results in~\cite{feldman2020hiding,balle2019privacy} have the privacy bound $\epsilon=\epsilon_0$ and $\delta=0$. For example, when total number of clients $n=10^6$, LDP parameter $\epsilon_0=3$, and we choose uniformly at random $k=1000$ clients at each iteration, then the conditions in \eqref{eqn:condition_clones} and \eqref{eqn:condition_blanket} do not hold.

 In Figure~\ref{fig:DP_composition_additional}, we plot our bound on the approximate
$\left(\epsilon,\delta\right)$-DP for a composition of $T$ mechanisms $\left(\calM_1,\ldots,\calM_T\right)$, where $\calM_t$ is a subsampled shuffle mechanism for $t\in \left[T\right]$. In all our experiments reported in Figure~\ref{fig:DP_composition_additional}, we fix $\delta=10^{-8}$ and $\epsilon_0=3$ and the conditions~\eqref{eqn:condition_clones}-\eqref{eqn:condition_blanket} do not hold with the parameter setting in these experiments. Thus, we compare our results with the bound $\epsilon=\epsilon_0$ and $\delta=0$.

 \begin{table}[t]
\centering
\begin{tabular}{ |c  |c | }
\hline
  Layer & Parameters  \\ 
 \hline\hline
Convolution & $16$ filters of $8\times 8$, Stride $2$\\ 
Max-Pooling & $2\times 2$\\ 
Convolution & $16$ filters of $4\times 4$, Stride $2$\\ 
Max-Pooling & $2\times 2$\\ 
Fully connected & $32$ units\\
Softmax& $10$ units\\
\hline
\end{tabular}
\caption{Model Architecture for MNIST}\label{T1}
\end{table}

\paragraph{Distributed private learning:} We numerically evaluate the proposed privacy-learning performance on training machine learning models. We consider the standard MNIST handwritten digit dataset that has $60,000$ training images and $10,000$ test images. We train a simple neural network that was also used in~\cite{ESA,papernot2020tempered} and described in Table~\ref{T1}. This model has $d=13,170$ parameters and achieves an accuracy of $99\%$ for non-private, uncompressed vanilla SGD. We assume that we have $n=60,000$ clients, where each client has one sample.

\begin{figure}[t]
\centering
\begin{subfigure}{0.45\textwidth}
    \centering 
  \includegraphics[scale=0.45]{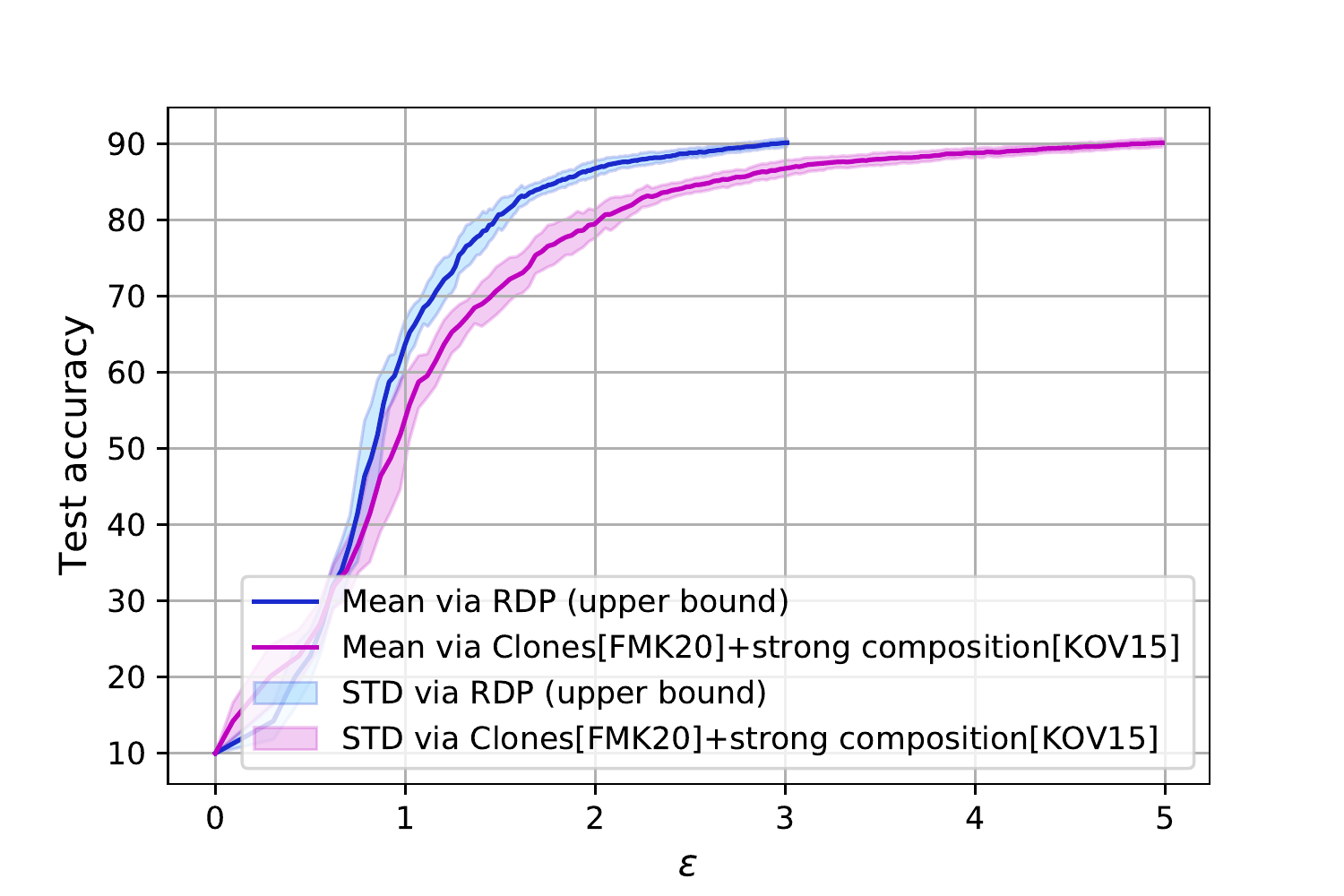}
  \caption{Privacy-Utility trade-offs on the MNIST dataset with $\ell_{\infty}$-norm clipping.}
  \label{fig:Fig4_3}
\end{subfigure}\hfil 
\begin{subfigure}{0.45\textwidth}
    \centering 
  \includegraphics[scale=0.43]{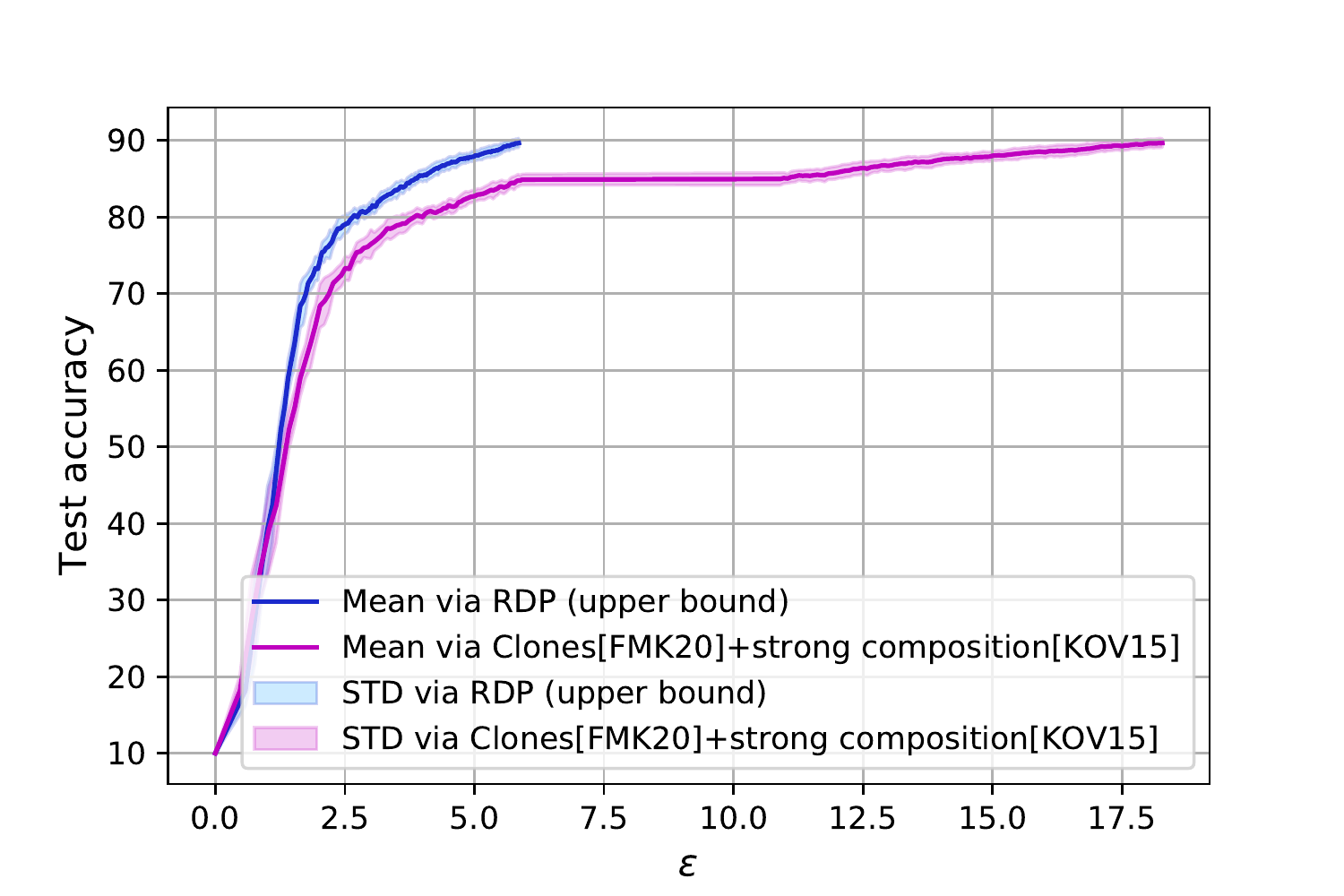}
  \caption{Privacy-Utility trade-offs on the MNIST dataset with $\ell_{2}$-norm clipping.}
  \label{fig:Fig4_4}
\end{subfigure}
\caption{Distributed private learning.}
\label{fig:DP_learning}
\end{figure}

In Figure~\ref{fig:Fig4_3}, we choose uniformly at random $10,000$ clients at each step of the Algorithm~\ref{algo:optimization-algo}, where each client clips the $\ell_{\infty}$-norm of the gradient with clipping parameter $C=1/100$ and applies the $\calR_{\infty}$ $\eps_0$-LDP mechanism proposed in~\cite{girgis2021shuffled-aistats} with $\epsilon_0=1.5$. We run Algorithm~\ref{algo:optimization-algo} with $\delta=10^{-5}$ for $200$ epochs, with learning rate $\eta=0.3$ for the first $70$ epochs, and then decrease it to $0.18$ in the remaining epochs. Figure~\ref{fig:Fig4_3} plots the mean and the standard deviation of privacy-accuracy trade-offs averaged over $10$ runs for $\ell_{\infty}$-norm clipping. 

In Figure~\ref{fig:Fig4_4}, we choose uniformly at random $2,000$ clients at each step of the Algorithm~\ref{algo:optimization-algo}, where each client clips the $\ell_{2}$-norm of the gradient with clipping parameter $C=0.005$ and applies the $\calR_{2}$ $\eps_0$-LDP mechanism (PrivUnit) proposed in~\cite{duchi2013local} with $\epsilon_0=2$. We run Algorithm~\ref{algo:optimization-algo} with $\delta=10^{-5}$ for $200$ epochs, with learning rate $\eta=12$ for the first $30$ epochs, and then decrease it to $4$ in the next $30$ epochs. We decrease the learning rate to $3.5$ for the remaining epochs. Figure~\ref{fig:Fig4_4} plots the mean and the standard deviation of privacy-accuracy trade-offs averaged over $4$ runs for $\ell_{2}$-norm clipping.

For our privacy analysis, the total privacy budget is computed by optimizing over RDP order $\lambda$ using our upper bound given in Theorem~\ref{thm:general_case}. For privacy analysis of~\cite{feldman2020hiding}, we first compute the privacy amplification by shuffling numerically given in~\cite{feldman2020hiding}; then we compute its privacy obtained when amplified via subsampling~\cite{Jonathan2017sampling}; and finally we use the strong composition theorem~\cite{kairouz2015composition} to obtain the central privacy parameter $\epsilon$.  

For the $\ell_{\infty}$-norm clipping, we achieve an accuracy of $80\%(\pm 1.8)$ with a total privacy budget of $\epsilon= 1.4$ using our new privacy analysis, whereas, \cite{feldman2020hiding} achieves an accuracy of $70.7\%(\pm 2.1)$ with the same privacy budget of $\epsilon= 1.4$ using the standard composition theorems. Furthermore, we achieve an accuracy of $90\%(\pm 0.5)$ with total privacy budget of $\epsilon= 2.91$ using our new privacy analysis, whereas, \cite{feldman2020hiding} (together with the standard strong composition theorem) achieve the same accuracy with a total privacy budget of $\epsilon= 4.82$. Similarly, for $\ell_{2}$-norm clipping, we achieve an accuracy of $81.15\%(\pm 0.7)$ with a total privacy budget of $\epsilon= 3$ using our new privacy analysis, whereas, \cite{feldman2020hiding} achieves an accuracy of only $76.46\%(\pm 1.9)$ with the same privacy budget of $\epsilon= 3$ using the standard composition theorems. Furthermore, we achieve an accuracy of $89.7\%(\pm 0.5)$ with a total privacy budget of $\epsilon= 5.8$ using our new privacy analysis, whereas, \cite{feldman2020hiding} (together with the standard strong composition theorem) achieve the same accuracy with a total privacy budget of $\epsilon= 18.3$.

%% file: general_case.tex
\section{Proof of Theorem~\ref{thm:general_case}: Upper Bound}\label{sec:general_case}

For any dataset $\calD_k=(d_1,\hdots,d_k)\in\calX^k$ containing of $k$ data points, we define a shuffle mechanism $\calM_{sh}(\calD_k)$ as follows:
\begin{equation}
\calM_{sh}(\calD_k)=\calH_{k}\left(\calR\left(d_1\right),\ldots,\calR\left(d_k\right)\right),
\end{equation}
where $\calH_k$ takes $k$ inputs and outputs a uniformly random permutation of them.
Recall from \eqref{shuffle-mech}, for any dataset $\calD_n=(d_1,\hdots,d_n)\in\calX^n$ containing $n$ data points, the subsampled-shuffle mechanism is defined as $\calM\left(\calD\right) := \calH_{k}\circ \samp_{k}^{n}\left(\calR\left(d_1\right),\ldots,\calR\left(d_n\right)\right)$.

The proof of Theorem~\ref{thm:general_case} consists of two steps. First, we bound the ternary-$|\chi|^{\alpha}$-DP of the shuffle mechanism $\calM_{sh}$ (see Theorem~\ref{thm:ternary_DP_shuffle}), which is the main technical contribution in this proof. Then, using this, we bound the RDP of the subsampled shuffle mechanism $\calM$.

\begin{theorem}[$\zeta$-ternary-$|\chi|^{\alpha}$-DP of the shuffle mechanism $\calM_{sh}$]\label{thm:ternary_DP_shuffle} 
For any integer $k\geq 2$, $\epsilon_0>0$, and all $\alpha\geq 2$, the $\zeta$-ternary-$|\chi|^{\alpha}$-DP of the shuffle mechanism $\calM_{sh}$ is bounded by:
\begin{equation}
\zeta\left(\alpha\right)^{\alpha}\leq \left\{\begin{array}{ll}
4\frac{\left(e^{\epsilon_0}-1\right)^2}{\overline{k}e^{\epsilon_0}}+(e^{\epsilon_0}-e^{-\epsilon_0})^{\alpha}e^{-\frac{k-1}{8e^{\epsilon_0}}}& \text{if\ } \alpha=2, \\
\alpha\Gamma\left(\alpha/2\right)\left(\frac{2\left(e^{2\epsilon_0}-1\right)^2}{\overline{k}e^{2\epsilon_0}}\right)^{\alpha/2}+(e^{\epsilon_0}-e^{-\epsilon_0})^{\alpha}e^{-\frac{k-1}{8e^{\epsilon_0}}}& \text{otherwise}, 
\end{array}
\right.
\end{equation} 
where $\overline{k}=\floor{\frac{k-1}{2e^{\epsilon_0}}}+1$ and $\Gamma\left(z\right)=\int_{0}^{\infty}x^{z-1}e^{-x}dx$ is the Gamma function. 
\end{theorem}
Theorem~\ref{thm:ternary_DP_shuffle} is one of the core technical results of this paper, and we prove it in Section~\ref{sec:proof_ternary_DP_shuffle}. 

It was shown in \cite[Proposition~$16$]{wang2019subsampled} that if a mechanism obeys $\zeta$-ternary-$|\chi|^{\alpha}$-DP, then its subsampled version (with subsampling parameter $\gamma$) will obey $\gamma\zeta$-ternary-$|\chi|^{\alpha}$-DP. Using that result, the authors then bounded the RDP of the subsampled mechanism in \cite[Eq.~$(9)$]{wang2019subsampled}.
Adapting that result to our setting, we have the following lemma.
\begin{lemma}[From $\zeta$-ternary-$|\chi|^{\alpha}$-DP to subsampled RDP]\label{lem:samplig} 
Suppose the shuffle mechanism $\calM_{sh}$ obeys $\zeta$-ternary-$|\chi|^{\alpha}$-DP. For any $\lambda\geq 2,k\leq n$, RDP of the subsampled shuffle mechanism $\calM$ (with subsampling parameter $\gamma=k/n$) is bounded by:
$\epsilon(\lambda)\leq \frac{1}{\lambda-1}\log\big(1+\sum_{\alpha=2}^{\lambda}\binom{\lambda}{\alpha} \gamma^{\alpha}\zeta(\alpha)^{\alpha}\big)$.
\end{lemma}
Lemma~\ref{lem:samplig} can be seen as a corollary to \cite[Proposition~$16$ and Eq.~$(9)$]{wang2019subsampled}. However, for completeness, we prove it in Appendix~\ref{app_sec:sampling}. Substituting the bound on $\zeta(\alpha)$ from Theorem~\ref{thm:ternary_DP_shuffle} into Lemma~\ref{lem:samplig}, we get 
\begin{align}
\epsilon\left(\lambda\right) %&\leq \frac{1}{\lambda-1}\log\left(1+\sum_{j=2}^{\lambda}\binom{\lambda}{j}\gamma^{j}\zeta\left(j\right)^{j}\right)\\
&\leq \frac{1}{\lambda-1}\log\left[1+\binom{\lambda}{2}\gamma^{2}\(4\frac{\left(e^{\epsilon_0}-1\right)^2}{\overline{k}e^{\epsilon_0}}+(e^{\epsilon_0}-e^{-\epsilon_0})^{2}e^{-\frac{k-1}{8e^{\epsilon_0}}}\) \right.\notag\\
&\hspace{2.5cm} \left.+\sum_{\alpha=3}^{\lambda}\binom{\lambda}{\alpha}\gamma^{\alpha} \( \alpha \Gamma(\alpha/2) \left(\frac{2\left(e^{2\epsilon_0}-1\right)^2}{\overline{k}e^{2\epsilon_0}}\right)^{\alpha/2}+(e^{\epsilon_0}-e^{-\epsilon_0})^{\alpha}e^{-\frac{k-1}{8e^{\epsilon_0}}}\)\right] \notag\\
&=\frac{1}{\lambda-1}\log\left[1+4\binom{\lambda}{2}\gamma^{2}\frac{\left(e^{\epsilon_0}-1\right)^2}{\overline{k}e^{\epsilon_0}}+\sum_{\alpha=3}^{\lambda}\binom{\lambda}{\alpha}\gamma^{\alpha} \alpha \Gamma(\alpha/2) \left(\frac{2\left(e^{2\epsilon_0}-1\right)^2}{\overline{k}e^{2\epsilon_0}}\right)^{\alpha/2}+\Upsilon\right] \label{final_bound_general_case},
\end{align} 
where $\Upsilon=\sum_{\alpha=2}^{\lambda} \binom{\lambda}{\alpha}\gamma^{\alpha}(e^{\epsilon_0}-e^{-\epsilon_0})^{\alpha}e^{-\frac{k-1}{8e^{\epsilon_0}}}=\left(\left(1+\gamma \frac{e^{2\epsilon_0}-1}{e^{\epsilon_0}}\right)^{\lambda}-1-\lambda\gamma \frac{e^{2\epsilon_0}-1}{e^{\epsilon_0}}\right)e^{-\frac{k-1}{8e^{\epsilon_0}}}$. 

The above expression in \eqref{final_bound_general_case} is the bound given in Theorem~\ref{thm:general_case}.

%% file: ternary_DP_shuffle.tex
\section{Proof of Theorem~\ref{thm:ternary_DP_shuffle}: Ternary $|\chi|^{\alpha}$-DP of the Shuffle Model}\label{sec:proof_ternary_DP_shuffle}

The proof has two main steps. In the first step, we reduce the problem of deriving ternary divergence for arbitrary neighboring datasets to the problem of deriving the ternary divergence for specific neighboring datasets, $\calD\sim\calD'\sim\calD''$, where all elements in $\calD$ are the same and $\calD',\calD''$ differ from $\calD$ in one entry. In the second step, we derive the ternary divergence for the special neighboring datasets.

The specific neighboring datasets to which we reduce our general problem has the following form:
\begin{align}
\calD_{\same}^{m} &= \left\lbrace (\calD_{m},\calD'_{m},\calD''_{m}): \calD_{m}= (d,\ldots,d,d)\in\calX^{m},\ \calD'_{m}=(d,\ldots,d,d')\in\calX^{m}, \text{ and }\notag \right. \\
&\left.\hspace{4cm} \calD''_{m}=(d,\ldots,d,d'')\in\calX^{m}, \text{ where } d,d',d''\in\calX \right\rbrace, \label{datasets-same}
\end{align}
Consider arbitrary neighboring datasets $\calD=(d_1,\ldots,d_{k-1},d_k)$, $\calD'=(d_1,\ldots,d_{k-1},d'_k)$, $\calD''=(d_1,\ldots,d_{k-1},d''_k)$, each having $k$ elements. For any $m\in\lbrace 0,\ldots,k-1\rbrace$, we define new neighboring datasets $\calD_{m+1}^{(k)}=(d''_k,\ldots,d''_{k},d_k)$, $\calD'^{(k)}_{m+1}=(d''_k,\ldots,d''_k,d'_k)$, and $\calD''^{(k)}_{m+1}=(d''_k,\ldots,d''_k)$, each having $m+1$ elements. Observe that $(\calD''^{(k)}_{m+1},\calD'^{(k)}_{m+1},\calD^{(k)}_{m+1})\in\calD_{\same}^{m}$. The first step of the proof is given in the following theorem:
\begin{theorem}[Reduction to the Special Case]\label{thm:reduce_special_case} 
Let $q=\frac{1}{e^{\epsilon_0}}$. We have: 
\begin{align}
&\mathbb{E}_{\bh\sim\calM_{sh}(\calD'')}\left[\left|\frac{\calM_{sh}(\calD)(\bh)-\calM_{sh}(\calD')(\bh)}{\calM_{sh}(\calD'')(\bh)}\right|^{\alpha}\right] \notag \\
&\hspace{1cm}\leq \mathbb{E}_{m\sim\emph{Bin}\left(k-1,q\right)}\left[\mathbb{E}_{\bh\sim\calM_{sh}(\calD_{m+1}''^{(k)})}\left[\left|\frac{\calM_{sh}(\calD_{m+1}^{(k)})(\bh)-\calM_{sh}(\calD_{m+1}'^{(k)})(\bh)}{\calM_{sh}(\calD_{m+1}''^{(k)})(\bh)}\right|^{\alpha}\right]\right]. \label{eq:reduce_special-case_bound}
\end{align}
\end{theorem}
We present a proof of Theorem~\ref{thm:reduce_special_case} in Section~\ref{app_reduce_special_case}. We know (by Chernoff bound) that the binomial r.v. is concentrated around its mean, which implies that the terms in the RHS of \eqref{eq:reduce_special-case_bound} that correspond to $m<(1-\tau)q(k-1)$ (we will take $\tau=1/2$) will contribute in a negligible amount. Then we show %in Lemma~\ref{lem:E_m-decreasing} (on page~\pageref{lem:E_m-decreasing}) 
that $E_{m}:=\mathbb{E}_{\bh\sim\calM_{sh}(\calD_{m+1}''^{(k)})}\left[\left|\frac{\calM_{sh}(\calD_{m+1}^{(k)})(\bh)-\calM_{sh}(\calD_{m+1}'^{(k)})(\bh)}{\calM_{sh}(\calD_{m+1}''^{(k)})(\bh)}\right|^{\alpha}\right]$ is a non-increasing function of $m$. These observation together imply that the RHS in \eqref{eq:reduce_special-case_bound} is approximately upper bounded by $E_{(1-\tau)q(k-1)}$.

Since $E_m$ is precisely what is required to bound the ternary DP for the specific neighboring datasets, we have reduced the problem of computing the ternary DP for arbitrary neighboring datasets to the problem of computing ternary DP for specific neighboring datasets. The second step of the proof bounds $E_{(1-\tau)q(n-1)}$, which follows from the result below that holds for any $m\in\bbN$.
\begin{theorem}[$|\chi|^{\alpha}$-DP for special case]\label{thm:ternary_special_case}
For any $m\in\mathbb{N}$, integer $\alpha\geq 2$, and $(\calD''_{m},\calD'_{m},\calD_{m})\in\calD_{\same}^{m}$, %we have
\begin{align*}
\mathbb{E}_{\bh\sim\calM_{sh}(\calD_m)}\left[\left|\frac{\calM_{sh}(\calD'_m)(\bh)-\calM_{sh}(\calD''_m)(\bh)}{\calM_{sh}(\calD_m)(\bh)}\right|^{\alpha}\right]\leq
\left\{\begin{array}{ll}
4\frac{\left(e^{\epsilon_0}-1\right)^2}{me^{\epsilon_0}}& \text{if\ } \alpha=2, \\
\alpha\Gamma(\alpha/2)\left(\frac{2(e^{2\epsilon_0}-1)^2}{me^{2\epsilon_0}}\right)^{\alpha/2}& \text{otherwise}. 
\end{array}
\right.
\end{align*}
\end{theorem}
The proof of Theorem~\ref{thm:ternary_special_case} is presented in Section~\ref{sec:app_ternary_special_case}. Missing details of how Theorem~\ref{thm:ternary_DP_shuffle} follows from Theorems~\ref{thm:reduce_special_case}, \ref{thm:ternary_special_case} can be found in Appendix~\ref{subsec_app:completeting-proof}.

\subsection{Proof of Theorem~\ref{thm:reduce_special_case}: Reduction to the Special Case}\label{app_reduce_special_case} 
First, we prove the joint-convexity of the ternary $|\chi|^{\alpha}$-divergence as it is important in the following proof.
\begin{lemma}[Joint-convexity of the ternary $|\chi|^{\alpha}$-divergence]\label{lemm:convextiy_ternary_div}
For all $\alpha\geq 1$, the ternary-$|\chi|^{\alpha}$-divergence $\mathbb{E}\left[\left|\frac{P-Q}{R}\right|^{\alpha}\right]$ is jointly convex in $P,Q$ and $R$. In other words, if $P_{a}=a P_0+(1-a)P_1$, $Q_{a}=a Q_0+(1-a)Q_1$, and $R_{a}=a R_0+(1-a)R_1$ for some $a\in\left[0,1\right]$, then the following holds
\begin{equation}
\mathbb{E}\left[\left|\frac{P_a-Q_a}{R_a}\right|^{\alpha}\right]\leq a \mathbb{E}\left[\left|\frac{P_0-Q_0}{R_0}\right|^{\alpha}\right]+(1-a)\mathbb{E}\left[\left|\frac{P_1-Q_1}{R_1}\right|^{\alpha}\right]
\end{equation}  
\end{lemma}
\begin{proof}
First, observe that $g(x,y)=|x-y|$ is jointly convex on $\mathbb{R}^2$, i.e., if $x_a=ax_0+(1-a)x_1$ and $y_a=ay_0+(1-a)y_1$, we have
\begin{align}
|x_a-y_a|&=|a(x_0-y_0)+(1-a)(x_1-y_1)| \notag \\
&\leq a |x_0-y_0|+(1-a)|x_1-y_1| \label{eqn:abs_inequality}
\end{align}
Let $f(x,y)=x^j/y^{j-1}$, which is jointly convex on $\mathbb{R}_{+}^{2}$ for $j\geq1$; see \cite[Lemma~$20$]{wang2019subsampled} for a proof. Thus, we get
\begin{equation}
\begin{aligned}
\frac{|P_a-Q_a|^j}{R_a^{j-1}}&\leq \frac{\left(a|P_0-Q_0|+(1-a)|P_1-Q_1|\right)^j}{(aR_0+(1-a)R_1)^{j-1}}
&\leq a \frac{|P_0-Q_0|^j}{R_0^{j-1}} +(1-a)\frac{|P_1-Q_1|^j}{R_1^{j-1}},
\end{aligned}
\end{equation}
where the first inequality is obtained from~\eqref{eqn:abs_inequality} and the second inequality is obtained from the convexity of $f(x,y)$.
\end{proof}

Now, we prove Theorem~\ref{thm:reduce_special_case}. Our proof is an adaptation of the proof of ~\cite[Theorem~$4$]{girgis2021renyi}. The difference comes from the fact that \cite[Theorem~$4$]{girgis2021renyi} was for Renyi divergence, whereas, here we are working with ternary $|\chi|^{\alpha}$-divergence. This changes some details and we provide a full proof of Theorem~\ref{thm:reduce_special_case} below.

 Let $\bp_i:=(p_{i1},\ldots,p_{iB})$, $\bp'_k:=(p'_{k1},\ldots,p'_{kB})$, $\bp''_k:=(p''_{k1},\ldots,p''_{kB})$ denote the probability distributions over $\calY$ when the input to $\calR$ is $d_i$, $d'_k$, and $d''_k$ respectively, where $p_{ij}=\Pr[\calR(d_i)=j]$ for all $j\in[B]$ and $i\in\left[n\right]$. 
Let $\calP=\lbrace \bp_i:i\in\left[k\right] \rbrace$, $\calP'=\lbrace \bp_i:i\in\left[k-1\right] \rbrace\bigcup \lbrace \bp'_k\rbrace$, and $\calP''=\lbrace \bp_i:i\in\left[k-1\right] \rbrace\bigcup \lbrace \bp''_k\rbrace$. 

For $i\in[k-1]$, let $\calP_{-i}=\calP\setminus\{\bp_i\}$ and also $\calP_{-k}=\calP\setminus\{\bp_k\}$. Here, $\calP,\calP',\calP''$ correspond to the datasets $\calD=\{d_1,\hdots,d_k\},\calD'=\{d_1,\hdots,d_{k-1},d'_k\}$, and $\calD''=\{d_1,\hdots,d_{k-1},d''_k\}$ respectively, and for any $i\in[k]$, $\calP_{-i}$ corresponds to the dataset $\calD_{-i}=\{d_1,\hdots,d_{i-1},d_{i+1},\hdots,d_k\}$.

For any collection $\calP=\{\bp_1,\hdots,\bp_k\}$ of $k$ distributions, we define $F(\calP)$ to be the distribution over $\calA_B^k$ (which is the set of histograms on $B$ bins with $k$ elements as defined in \eqref{histogram-set}) that is induced when every client $i$ (independent to the other clients) samples an element from $[B]$ accordingly to the probability distribution $\bp_i$. 
Formally, for any $\bh\in\calA_B^k$, define
\begin{align}\label{mapping-possibilities}
\calU_{\bh} := \left\{ (\calU_1,\hdots,\calU_B): \calU_1,\hdots,\calU_B\subseteq[k] \text{ s.t. } \bigcup_{j=1}^B\calU_j=[k] \text{ and } |\calU_j|=h_j,\forall j\in[B] \right\}.
\end{align}
Note that for each $(\calU_1,\hdots,\calU_B)\in\calU_{\bh}$, $\calU_j$ for $j=1,\hdots,B$ denotes the identities of the clients that map to the $j$'th element in $[B]$ -- here $\calU_j$'s are disjoint for all $j\in[B]$. Note also that $|\calU_{\bh}|=\binom{k}{\bh}=\frac{k!}{h_1!h_2!\hdots h_B!}$. It is easy to verify that for any $\bh\in\calA_B^k$, $F(\calP)(\bh)$ is equal to
\begin{align}\label{general-distribution}
F(\calP)(\bh) = \sum_{(\calU_1,\hdots,\calU_B)\in\calU_{\bh}}\prod_{j=1}^B\prod_{i\in\calU_j}p_{ij}
\end{align}
Similarly, we can define $F(\calP'),F(\calP''),F(\calP_{-i}),F(\calP'_{-i})$, and $F(\calP''_{-i})$. 
Note that $F(\calP)$, $F(\calP')$ and $F(\calP'')$ are distributions over $\calA_B^k$, whereas, $F(\calP_{-i})$, $F(\calP'_{-i})$, and $F(\calP''_{-i})$ are distributions over $\calA_B^{k-1}$.
It is easy to see that $F(\calP)=\calM_{sh}(\calD)$, $F(\calP')=\calM_{sh}(\calD')$, and $F(\calP'')=\calM_{sh}(\calD'')$.

A crucial observation is that any distribution $\bp_i$ can be written as the following mixture distribution:
\begin{equation}\label{eq:mixture-dist}
\bp_i= q \bp''_k+\left(1-q\right)\tilde{\bp}_i,
\end{equation}
where $q=\frac{1}{e^{\epsilon_0}}$. The distribution $\tilde{\bp}_i=\left[\tilde{p}_{i1},\ldots,\tilde{p}_{iB}\right]$ is given by $\tilde{p}_{ij}=\frac{p_{ij}-q p''_{kj}}{1-q}$, where it is easy to verify that $\tilde{p}_{ij}\geq 0$ and $\sum_{j=1}^{B}\tilde{p}_{ij}=1$. The idea of writing the distribution of the output of an LDP mechanism as a mixture distribution was previously proposed in~\cite{balle2019privacy,feldman2020hiding}. However, the way these mixture distributions are used in our RDP analysis is quite different from their use in studying the hockey-stick divergence.

For any $\calC\subseteq[k-1]$, define three sets $\calP_{\calC},\calP'_{\calC}$, and $\calP''_{\calC}$ having $k$ distributions each, as follows:
\begin{align}
\calP_{\calC} &= \{\hat{\bp}_1,\hdots,\hat{\bp}_{k-1}\}\bigcup\{\bp_k\}, \label{eq:defn_P_C} \\
\calP'_{\calC} &= \{\hat{\bp}_1,\hdots,\hat{\bp}_{k-1}\}\bigcup\{\bp'_k\}, \label{eq:defn_P_C-prime}\\
\calP''_{\calC} &= \{\hat{\bp}_1,\hdots,\hat{\bp}_{k-1}\}\bigcup\{\bp''_k\}, \label{eq:defn_P_C-dprime}
\end{align}
where, for every $i\in[k-1]$, $\hat{\bp}_i$ is defined as follows:
\begin{equation}\label{eq:defn_hatP}
\hat{\bp}_i=
\begin{cases}
\bp''_k & \text{ if } i\in\calC, \\
\tilde{\bp}_i & \text{ if } i\in[k-1]\setminus\calC.
\end{cases}
\end{equation}
In the following lemma, we show that $F(\calP),F(\calP')$, and $F(\calP'')$ can be written as convex combinations of $\{F(\calP_{\calC}):\calC\subseteq[k-1]\},\{F(\calP'_{\calC}):\calC\subseteq[n-1]\}$, and $\{F(\calP''_{\calC}):\calC\subseteq[k-1]\}$, respectively, where for any $\calC\subseteq[k-1]$, $F(\calP_{\calC}),F(\calP'_{\calC})$. and $F(\calP''_{\calC})$ can be computed analogously as in \eqref{general-distribution}.
\begin{lemma} [Mixture Interpretation {\cite[Lemma $3$]{girgis2021renyi}}]\label{lem:convex-combinations}
$F(\calP),F(\calP')$, and $F(\calP'')$ can be written as the following convex combinations:
\begin{align}
F(\calP)&=\sum_{\calC\subseteq [k-1]} q^{|\calC|}(1-q)^{k-|\calC|-1}F(\calP_{\calC}), \label{P_mixture} \\
F(\calP')&=\sum_{\calC\subseteq [k-1]} q^{|\calC|}(1-q)^{k-|\calC|-1}F(\calP'_{\calC}), \label{P-prime_mixture} \\
F(\calP'')&=\sum_{\calC\subseteq [k-1]} q^{|\calC|}(1-q)^{k-|\calC|-1}F(\calP''_{\calC}), \label{P-d-prime_mixture}
\end{align}
where $\calP_{\calC},\calP_{\calC}',\calP_{\calC}''$ are defined in \eqref{eq:defn_P_C}-\eqref{eq:defn_hatP}.
\end{lemma} 

From Lemma~\ref{lemm:convextiy_ternary_div} and Lemma~\ref{lem:convex-combinations}, we get
\begin{align}
&\mathbb{E}_{\bh\sim F(\calP'')}\left[\left|\frac{F(\calP)(\bh)-F(\calP')(\bh)}{F(\calP'')(\bh)}\right|^{\alpha}\right] \notag \\
&\hspace{2cm}\leq \sum_{\calC\subseteq \left[k-1\right]} q^{|\calC|}\left(1-q\right)^{k-|\calC|-1} \mathbb{E}_{\bh\sim F\left(\calP''_{\calC}\right)}\left[\left|\frac{F(\calP_{\calC})(\bh)-F(\calP'_{\calC})(\bh)}{F(\calP''_{\calC})(\bh)}\right|^{\alpha}\right]. \label{eqn:rdp_bound}
\end{align}
For any $\calC\subseteq[k-1]$, let $\widetilde{\calP}_{[k-1]\setminus\calC}=\{\tilde{\bp}_i:i\in\left[k-1\right]\setminus \calC\}$. With this notation, note that $\calP_{\calC}\setminus\widetilde{\calP}_{[k-1]\setminus\calC}=\{\bp''_k,\hdots,\bp''_k\} \bigcup \{\bp_k\}$, $\calP'_{\calC}\setminus\widetilde{\calP}_{[k-1]\setminus\calC}=\{\bp''_k,\hdots,\bp''_k\} \bigcup \{\bp'_k\}$, and $\calP''_{\calC}\setminus\widetilde{\calP}_{[k-1]\setminus\calC}=\{\bp''_k,\hdots,\bp''_k\} \bigcup \{\bp''_k\}$ are a triple of specific neighboring distributions, each containing $|\calC|+1$ distributions. In other words, if we define $\calD_{|\calC|+1}^{(k)}=\left(d''_k,\ldots,d''_k,d_{k}\right)$, $\calD_{|\calC|+1}'^{(k)}=\left(d''_k,\ldots,d''_k,d_k'\right)$, and $\calD_{|\calC|+1}''^{(k)}=\left(d''_k,\ldots,d''_k,d_k''\right)$, each having $(|\calC|+1)$ data points (note that $(\calD_{|\calC|+1}''^{(k)},\calD_{|\calC|+1}'^{(k)},\calD_{|\calC|+1}^{(k)})\in\calD_{\same}^{|\calC|+1}$), then the mechanisms $\calM_{sh}(\calD_{|\calC|+1}^{(k)})$, $\calM_{sh}(\calD_{|\calC|+1}'^{(k)})$, and $\calM_{sh}(\calD_{|\calC|+1}''^{(k)})$ will have distributions $F(\calP_{\calC}\setminus\widetilde{\calP}_{[k-1]\setminus\calC})$, $F(\calP'_{\calC}\setminus\widetilde{\calP}_{[k-1]\setminus\calC})$, and $F(\calP''_{\calC}\setminus\widetilde{\calP}_{[k-1]\setminus\calC})$, respectively.

Now, since $(\calD_{|\calC|+1}''^{(k)},\calD_{|\calC|+1}'^{(k)},\calD_{|\calC|+1}^{(k)})\in\calD_{\same}^{|\calC|+1}$, if we remove the effect of distributions in $\widetilde{\calP}_{[k-1]\setminus\calC}$ in the RHS of \eqref{eqn:rdp_bound}, we would be able to bound the RHS of \eqref{eqn:rdp_bound} using the ternary $|\chi|^{\alpha}$-divergence for the special neighboring datasets in $\calD_{\same}^{|\calC|+1}$. This is precisely what we will do in the following lemma and the subsequent corollary, where we will eliminate the distributions in $\widetilde{\calP}_{[k-1]\setminus\calC}$ in the RHS \eqref{eqn:rdp_bound}.

The following lemma holds for arbitrary triples $(\calP,\calP',\calP'')$ of neighboring distributions $\calP=\{\bp_1,\hdots,\bp_{k-1},\bp_k\}$, $\calP'=\{\bp_1,\hdots,\bp_{k-1},\bp'_k\}$, and $\calP''=\{\bp_1,\hdots,\bp_{k-1},\bp''_k\}$, where we show that the ternary $|\chi|^{\alpha}$-divergence $\mathbb{E}_{\bh\sim F(\calP'')}\left[\left|\frac{F(\calP)(\bh)-F(\calP')(\bh)}{F(\calP'')(\bh)}\right|^{\alpha}\right]$ does not decrease when we eliminate a distribution $\bp_i$ (i.e., remove the data point $d_i$ from the datasets) for any $i\in[k-1]$. 

\begin{lemma}[Monotonicity]\label{lem:cvx_tdp} 
For any $i\in\left[k-1\right]$, we have
\begin{equation}\label{eq:cvx_tdp} 
\mathbb{E}_{\bh\sim F\left(\calP''\right)}\left[\left|\frac{F\left(\calP\right)(\bh)-F\left(\calP'\right)(\bh)}{F\left(\calP''\right)(\bh)}\right|^{\alpha}\right]\leq \mathbb{E}_{\bh\sim F\left(\calP''_{-i}\right)}\left[\left|\frac{F\left(\calP_{-i}\right)(\bh)-F\left(\calP'_{-i}\right)(\bh)}{F\left(\calP''_{-i}\right)(\bh)}\right|^{\alpha}\right].
\end{equation}
\end{lemma}
\begin{proof}
This can be proved along the lines of the proof of \cite[Lemma~$5$]{girgis2021renyi}, which shows that $\mathbb{E}_{\bh\sim F\left(\calP'\right)}\left[\left(\frac{F\left(\calP\right)\left(\bh\right)}{F\left(\calP'\right)\left(\bh\right)}\right)^{\lambda}\right]\leq \mathbb{E}_{\bh\sim F\left(\calP'_{-i}\right)}\left[\left(\frac{F\left(\calP_{-i}\right)\left(\bh\right)}{F\left(\calP'_{-i}\right)\left(\bh\right)}\right)^{\lambda}\right]$ holds for all $i\in[k-1]$. This is a result about Renyi divergence, and the only property of the Renyi divergence that is used in the proof of \cite[Lemma~$5$]{girgis2021renyi} is that $\bbE_{\bh\sim F(\calP')}\left[\left(\frac{F(\calP)(\bh)}{F(\calP')(\bh)}\right)^{\lambda}\right]$ is convex in $\bp_i$ for any $i\in[k-1]$.

Note that Lemma~\ref{lem:cvx_tdp} is about the ternary $|\chi|^{\alpha}$-divergence, and the required convexity about this follows from Lemma~\ref{lemm:convextiy_ternary_div}. So, following the proof of \cite[Lemma~$5$]{girgis2021renyi} and using Lemma~\ref{lemm:convextiy_ternary_div}, proves Lemma~\ref{lem:cvx_tdp}.
\end{proof}

Now, for any given $\calC\subseteq [k-1]$, by eliminating the distributions $\tilde{\bp}_i$ in $\widetilde{\calP}_{[k-1]\setminus\calC}=\{\tilde{\bp}_i:i\in\left[k-1\right]\setminus \calC\}$ from $\calP_C$, $\calP'_C$, and $\calP''_C$ (by repeatedly applying Lemma~\ref{lem:cvx_tdp}), we get that
\begin{align}
\mathbb{E}_{\bh\sim F\left(\calP''_{\calC}\right)}&\left[\left|\frac{F\left(\calP_{\calC}\right)(\bh)-F\left(\calP'_{\calC}\right)(\bh)}{F\left(\calP''_{\calC}\right)(\bh)}\right|^{\alpha}\right] \notag \\
&\hspace{2cm}\leq \mathbb{E}_{\bh\sim \calM_{sh}(\calD''^{(k)}_{m+1})}\left[\left|\frac{\calM_{sh}(\calD^{(k)}_{m+1})(\bh)-\calM_{sh}(\calD'^{(k)}_{m+1})(\bh)}{\calM_{sh}(\calD''^{(k)}_{m+1})(\bh)}\right|^{\alpha}\right], \label{eqn:reduction_elimenate}
\end{align} 
where $m=|\calC|$. By substituting from~\eqref{eqn:reduction_elimenate} into~\eqref{eqn:rdp_bound} completes the proof of Theorem~\ref{thm:reduce_special_case}.

\subsection{Proof of Theorem~\ref{thm:ternary_special_case}: Ternary $|\chi|^{\alpha}$-DP of the Special Case}\label{sec:app_ternary_special_case} 

First, we present the following standard inequality which is important to our proof.

\begin{lemma}\label{lemm:traingle_inequality} 
Let $x,y\in\bbR$ be any two real numbers.
Then, for all $j\geq 1$, we have
\begin{equation}
|x+y|^{j} \leq 2^{j-1}\left(|x|^{j}+|y|^{j}\right).
\end{equation}
\end{lemma}
\begin{proof}
The proof is simple from the convexity of the function $f(x)=x^{j}$ for $j\geq1$.
\begin{align*}
|x+y|^{j} &= 2^{j}\left\vert\frac{x+y}{2}\right\vert^{j} \leq 2^{j}\left(\frac{|x|+|y|}{2}\right)^{j} \leq 2^{j}\left(\frac{|x|^j + |y|^j}{2}\right) = 2^{j-1}\(|x|^j + |y|^j\),
\end{align*}
where the second inequality is obtained from the Jensen's inequality and the fact that the function $f(x)=x^{j}$ is convex on $\mathbb{R}^{+}$ for all $j\geq 1$.
\end{proof}
%Now, we are ready to prove Theorem~\ref{thm:ternary_special_case}. 
From Lemma~\ref{lemm:traingle_inequality}, we get the following corollary.
\begin{corollary}\label{cor:triange_inequality} 
Fix an arbitrary $m\in\bbN$ and consider any three mutually neighboring datasets $\calD_{m},\ \calD'_{m},\ \calD''_{m}$, where $\calD_m=(d,\hdots,d)\in\calX^m$, $\calD'_m=(d,\hdots,d,d')\in\calX^m$ and $\calD''_m=(d,\hdots,d,d'')\in\calX^m$. The ternary $|\chi|^{\alpha}$-DP is bounded by
\begin{align}
&\mathbb{E}_{\bh\sim\calM_{sh}(\calD_m)}\left[\left|\frac{\calM_{sh}(\calD'_m)(\bh)-\calM_{sh}(\calD''_m)(\bh)}{\calM_{sh}(\calD_m)(\bh)}\right|^{\alpha}\right] \notag \\
&\quad\leq 2^{\alpha-1}\left( \mathbb{E}_{\bh\sim\calM_{sh}(\calD_m)}\left[\left|\frac{\calM_{sh}(\calD'_m)(\bh)}{\calM_{sh}(\calD_m)(\bh)}-1\right|^{\alpha}\right]+\mathbb{E}_{\bh\sim\calM_{sh}(\calD_m)}\left[\left|\frac{\calM_{sh}(\calD''_m)(\bh)}{\calM_{sh}(\calD_m)(\bh)}-1\right|^{\alpha}\right]\right). \label{eq:triange_inequality_cor}
\end{align}
\end{corollary}
\begin{proof}
Fix any $\bh\in\calA_B^m$, and take $x=\Big(\frac{\calM_{sh}(\calD'_m)(\bh)}{\calM_{sh}(\calD_m)(\bh)}-1\Big)$, $y=-\Big(\frac{\calM_{sh}(\calD''_m)(\bh)}{\calM_{sh}(\calD_m)(\bh)}-1\Big)$. Then applying Lemma~\ref{lemm:traingle_inequality} and taking expectation w.r.t.\ $\bh\sim\calM_{sh}(\calD_m)$ will yield Corollary~\ref{cor:triange_inequality}.
\end{proof}
\begin{remark}
Observe that the proof of Corollary~\ref{cor:triange_inequality} does not require $\calD_m,\calD_m',\calD_m''$ to be special triple of neighboring datasets such that $(\calD_m,\calD_m',\calD_m'')\in\calD_{\same}^m$. In fact, Corollary~\ref{cor:triange_inequality} holds for any triple of distributions $p,q,r$ over the same domain, for which we can show $\bbE_r[\left|\frac{p-q}{r}\right|^{\alpha}] \leq 2^{\alpha-1}\(\bbE_r[\left|\frac{p}{r}-1\right|^{\alpha}]+\bbE_r[\left|\frac{q}{r}-1\right|^{\alpha}]\)$.
\end{remark}
Now, in order to prove Theorem~\ref{thm:ternary_special_case}, it suffices to bound the expectation terms on the RHS of \eqref{eq:triange_inequality_cor}. This is what we do in the lemma below.
\begin{lemma}[{\hspace{-0.005cm}\cite[Lemma~$6$]{girgis2021renyi}}]\label{lemm:T_same_Bound}
For any pair of the special pair of neighboring datasets $\calD_{m},\calD'_{m}$, where $\calD_m=(d,\hdots,d)\in\calX^m$ and $\calD'_m=(d,\hdots,d,d')\in\calX^m$, we have
\begin{align*}
\mathbb{E}_{\bh\sim\calM_{sh}(\calD_m)}\left[\left|\frac{\calM_{sh}(\calD'_m)(\bh)}{\calM_{sh}(\calD_m)(\bh)}-1\right|^{\alpha}\right]\leq \left\{\begin{array}{ll}
\frac{\left(e^{\epsilon_0}-1\right)^2}{me^{\epsilon_0}}& \text{if\ } \alpha=2, \\
\alpha\Gamma(\alpha/2)\left(\frac{\left(e^{2\epsilon_0}-1\right)^2}{2me^{2\epsilon_0}}\right)^{\alpha/2}& \text{otherwise}. 
\end{array}
\right.
\end{align*}
\end{lemma}
Substituting the bound from Lemma~\ref{lemm:T_same_Bound} into Corollary~\ref{cor:triange_inequality}, we get 
\begin{align*}
\mathbb{E}_{\bh\sim\calM_{sh}(\calD_m)}\left[\left|\frac{\calM_{sh}(\calD'_m)(\bh)-\calM_{sh}(\calD''_m)(\bh)}{\calM_{sh}(\calD_m)(\bh)}\right|^{\alpha}\right]\leq
\left\{\begin{array}{ll}
4\frac{\left(e^{\epsilon_0}-1\right)^2}{me^{\epsilon_0}}& \text{if } \alpha=2, \\
\alpha\Gamma(\alpha/2)\left(\frac{2\left(e^{2\epsilon_0}-1\right)^2}{me^{2\epsilon_0}}\right)^{\alpha/2}& \text{otherwise},
\end{array}
\right.
\end{align*}
which completes the proof of Theorem~\ref{thm:ternary_special_case}.

%% file: app_lower_bound.tex
\section{Proof of Theorem~\ref{thm:lower_bound} (Lower Bound)}\label{app:lower-bound-proof}
Consider the binary case, where each data point $d$ can take a value from $\calX=\lbrace 0,1\rbrace$. Let the local randomizer $\calR$ be the binary randomized response (2RR) mechanism, where $\Pr\left[\calR\left(d\right)=d\right]=\frac{e^{\epsilon_0}}{e^{\epsilon_0}+1}$ for $d\in\calX$. It is easy to verify that $\calR$ is an $\eps_0$-LDP mechanism. For simplicity, let $p=\frac{1}{e^{\epsilon_0}+1}$. Consider two neighboring datasets $\calD,\ \calD' \in\{0,1\}^k$, where $\calD=\left(0,\ldots,0,0\right)$ and $\calD'=\left(0,\ldots,0,1\right)$. Let $m\in\left\{0,\ldots,k\right\}$ denote the number of ones in the output of the shuffler. We define two distributions
\begin{equation}
\begin{aligned}
\mu_0(m)&= \binom{k}{m} p^{m} (1-p)^{k-m}, \\
\mu_1(m)&= (1-p) \binom{k-1}{m-1} p^{m-1} (1-p)^{k-m}+ p\binom{k-1}{m} p^{m} (1-p)^{k-m-1}.
\end{aligned}
\end{equation}
As argued on page~\pageref{histogram-set}, since the output of the shuffled mechanism $\calM$ can be thought of as the distribution of the number of ones in the output, we have that $m\sim\calM(\calD)$ is distributed as a Binomial random variable Bin$(k,p)$. Thus, we have
\begin{align*}
\calM(\calD)(m)&= \mu_0(m) \\ 
\calM(\calD')(m)&= (1-\gamma)\mu_1(m)+\gamma \mu_0(m)
\end{align*}
It will be useful to compute $\frac{\mu_1(m)}{\mu_0(m)}-1$ for the calculations later.
\begin{align}
\frac{\mu_1(m)}{\mu_0(m)}-1 &= \frac{(1-p) \binom{k-1}{m-1} p^{m-1} (1-p)^{k-m} + p\binom{k-1}{m} p^{m} (1-p)^{k-m-1}}{\binom{k}{m} p^{m} (1-p)^{k-m}} -1 \notag \\
&= \frac{m}{k}\frac{(1-p)}{p} + \frac{(k-m)}{k}\frac{p}{(1-p)} -1 \notag \\
&= \frac{m}{k} e^{\eps_0} + \frac{(k-m)}{k} e^{-\eps_0} -1 \notag \\
&= \frac{m}{k}\(e^{\eps_0}-e^{-\eps_0}\) + e^{-\eps_0} -1 \notag \\
&= \frac{m}{k}\(\frac{e^{2\eps_0}-1}{e^{\eps_0}}\) - \(\frac{e^{\eps_0} -1}{e^{\eps_0}}\) \notag \\
&= \(\frac{e^{2\eps_0}-1}{ke^{\eps_0}}\)\(m - \frac{k}{e^{\eps_0}+1}\) \label{eq:compute_ratio_lb}
\end{align}
Thus, we have that 
\begin{align*}
\mathbb{E}_{m\sim\calM(\calD)}&\left[\left(\frac{\calM(\calD')(m)}{\calM(\calD)(m)}\right)^{\lambda}\right]=\mathbb{E}\left[\left(1+\gamma\left(\frac{\mu_1(m)}{\mu_0(m)}-1\right)\right)^{\lambda}\right]\\
&\stackrel{\text{(a)}}{=} 1+\sum_{i=1}^{\lambda} \binom{\lambda}{i} \gamma^{i} \mathbb{E}\left[\left(\frac{\mu_1(m)}{\mu_0(m)}-1\right)^{i}\right] \\
&\stackrel{\text{(b)}}{=} 1+\sum_{i=2}^{\lambda} \binom{\lambda}{i} \gamma^{i} \mathbb{E}\left[\left(\frac{\mu_1(m)}{\mu_0(m)}-1\right)^{i}\right] \\
%&=1+\sum_{i=2}^{\lambda} \binom{\lambda}{i}\gamma^i \mathbb{E}\left[\left(\frac{m}{k}e^{\epsilon_0}+\frac{(k-m)}{k}e^{-\epsilon_0}-1\right)^{i}\right]\\
&=1+\sum_{i=2}^{\lambda} \binom{\lambda}{i}\gamma^{i} \left(\frac{\left(e^{2\epsilon_0}-1\right)}{ke^{\epsilon_0}}\right)^{i} \mathbb{E}\left[\left(m-\frac{k}{e^{\epsilon_0}+1}\right)^{i}\right] \tag{from \eqref{eq:compute_ratio_lb}} \\
&\stackrel{\text{(c)}}{=} 1+\binom{\lambda}{2}\gamma^{2} \frac{\left(e^{\epsilon_0}-1\right)^2}{ke^{\epsilon_0}}+\sum_{i=3}^{\lambda} \binom{\lambda}{i} \gamma^{i}\left(\frac{\left(e^{2\epsilon_0}-1\right)}{ke^{\epsilon_0}}\right)^{i} \mathbb{E}\left[\left(m-\frac{k}{e^{\epsilon_0}+1}\right)^{i}\right].
\end{align*}
Here, step (a) from the polynomial expansion $(1+x)^k=\sum_{m=0}^{k}\binom{k}{m}x^m$, 
step (b) follows because the term corresponding to $i=1$ is zero (i.e., $\mathbb{E}_{m\sim\mu_0}\left[\left(\frac{\mu_1(m)}{\mu_0(m)}-1\right)\right]=0$),
and step (c) from the from the fact that $\mathbb{E}_{m\sim\mu_0}\left[\left(m-\frac{k}{e^{\epsilon_0}+1}\right)^{2}\right]=kp(1-p)=\frac{ke^{\eps_0}}{(e^{\eps_0}+1)^2}$, which is equal to the variance of the Binomial random variable. This completes the proof of Theorem~\ref{thm:lower_bound}.

%% file: app_optimization_performance.tex
\section{Proof of Theorem~\ref{thm:main-opt-result}: Privacy-Convergence Tradeoff}\label{app_sec:OptPerf}

In this section, we prove the privacy-convergence tradeoff of Algorithm~\ref{algo:optimization-algo} and prove Theorem~\ref{thm:main-opt-result}.

The privacy part is straightforward from conversion from RDP to approximate DP using Lemma~\ref{lem:RDP_DP} and Theorem~\ref{thm:general_case}. Now, we prove the convergence rate.

 At iteration $t\in\left[T\right]$ of Algorithm~\ref{algo:optimization-algo}, server averages the $k$ received gradients and obtains $\overline{\mathbf{g}}_t=\frac{1}{k}\sum_{i\in\calU_t}\mathcal{R}_p\left(\tilde{\mathbf{g}}_t\left(d_{i}\right)\right)$ and then updates the parameter vector as $\theta_{t+1}\gets \prod_{\mathcal{C}}\left( \theta_t -\eta_t \overline{\mathbf{g}}_t\right)$. 
Observe that the mechanism $\calR_p$ is unbiased and has a bounded variance: $\sup_{\mathbf{x}\in\calB_{p}\left(L\right)}\mathbb{E}\|\calR_p\left(\mathbf{x}\right)-\mathbf{x}\|_{2}^{2} \leq G_{p}^2(L)$. As a result, the average gradient $\overline{\mathbf{g}}_t$ is also unbiased, i.e., we have $\mathbb{E}\left[\overline{\mathbf{g}}_t\right]=\nabla_{\theta_t}F\left(\theta_t\right)$, where expectation is taken with respect to the random subsampling of clients as well as the randomness of the mechanism $\mathcal{R}_p$. Now we show that $\overline{\mathbf{g}}_t$ has a bounded second moment.
 \begin{lemma}\label{lem:2nd-moment-bound}
For any $d\in\calX$, if the function $f\left(\theta;.\right):\mathcal{C}\times \calX\to \mathbb{R}$ is convex and $L$-Lipschitz continuous with respect to the $\ell_g$-norm, which is the dual of the $\ell_p$-norm (i.e., $\frac{1}{p}+\frac{1}{g}=1$), then we have
\begin{align}
\mathbb{E}\|\overline{\mathbf{g}}_t\|_2^2 \leq L^2\max\lbrace d^{1-\frac{2}{p}},1\rbrace \( 1+ \frac{cd}{qn}\left(\frac{e^{\epsilon_0}+1}{e^{\epsilon_0}-1}\right)^2 \),
\end{align}
where $c$ is a global constant: $c=4$ if $p\in\{1,\infty\}$ and $c=14$ if $p\notin\{1,\infty\}$. 
\end{lemma}
\begin{proof}
Under the conditions of the lemma, we have from \cite[Lemma~$2.6$]{shalev2012online} that $\|\nabla_{\theta} f\left(\theta;d\right)\| \leq L$ for all $d\in\calX$, which implies that $\|\nabla_{\theta}F(\theta)\| \leq L$.
Thus, we have
\begin{align*}
&\mathbb{E}\|\overline{\mathbf{g}}_t\|_2^2=\|\mathbb{E}\left[\overline{\mathbf{g}}_t\right]\|_2^2+\mathbb{E}\|\overline{\mathbf{g}}_t-\mathbb{E}\left[\overline{\mathbf{g}}_t\right]\|_2^2\\
&\ \stackrel{\left(a\right)}{\leq} \max\lbrace d^{1-\frac{2}{p}},1\rbrace L^2+\mathbb{E}\|\overline{\mathbf{g}}_t-\mathbb{E}\left[\overline{\mathbf{g}}_t\right]\|_2^2 \\
&\ \stackrel{\left(b\right)}{\leq}\max\lbrace d^{1-\frac{2}{p}},1\rbrace L^2+\frac{G_{p}(L)^2}{k}\\
&\ \stackrel{\left(c\right)}{=}\max\lbrace d^{1-\frac{2}{p}},1\rbrace L^2+\frac{G_{p}^2(L)}{\gamma n},
\end{align*}
Step $\left(a\right)$ follows from the fact that $\|\nabla_{\theta_t}F\left(\theta_t\right)\| \leq L$ together with the norm inequality $\|\bu\|_q\leq\|\bu\|_p\leq d^{\frac{1}{p}-\frac{1}{q}}\|\bu\|_q$ for $1\leq p\leq q\leq\infty$. Step $\left(b\right)$ follows from the assumption that $\calR_p$ has bounded variance. Step (c) uses $\gamma=\frac{k}{n}$.
\end{proof}
Now, we can use standard SGD convergence results for convex functions. In particular, we use the following result from \cite{shamir2013stochastic}.
\begin{lemma}[SGD Convergence~\cite{shamir2013stochastic}]\label{lem:convergence_sgd} 
Let $F\left(\theta\right)$ be a convex function, and the set $\mathcal{C}$ has diameter $D$. Consider a stochastic gradient descent algorithm $\theta_{t+1}\gets \prod_{\mathcal{C}}\left( \theta_t-\eta_t \mathbf{g}_t\right)$, where $\mathbf{g}_t$ satisfies $\mathbb{E}\left[\mathbf{g}_t\right]=\nabla_{\theta_t}F\left(\theta_t\right)$ and $\mathbb{E}\|\mathbf{g}_t\|_{2}^{2} \leq G^{2}$. By setting $\eta_t=\frac{D}{G\sqrt{t}}$, we get
\begin{equation}
\mathbb{E}\left[F\left(\theta_{T}\right)\right] - F\left(\theta^{*}\right) \leq 2DG\frac{2+\log\left(T\right)}{\sqrt{T}}=\mathcal{O}\left(DG\frac{\log\left(T\right)}{\sqrt{T}}\right).
\end{equation}
\end{lemma}
As shown in Lemma~\ref{lem:2nd-moment-bound} and above that Algorithm~\ref{algo:optimization-algo} satisfies the premise of Lemma~\ref{lem:convergence_sgd}. Now, using the bound on $G^2$ from Lemma~\ref{lem:2nd-moment-bound}, we have that the output $\theta_T$ of Algorithm~\ref{algo:optimization-algo} satisfies
\begin{equation}\label{general_convergence-2}
\begin{aligned}
&\mathbb{E}\left[F\left(\theta_{T}\right)\right] - F\left(\theta^{*}\right) \leq \calO\left(DG\frac{\log\left(T\right)}{\sqrt{T}}\right),
\end{aligned}
\end{equation}
where $G^2 = \max\lbrace d^{1-\frac{2}{p}},1\rbrace L^2+\frac{G_{p}^2(L)}{\gamma n}$. This completes the proof of the second part of Theorem~\ref{thm:main-opt-result}.

%% file: app_preliminaries.tex
\section{Literature Review}\label{app_sec:LongerLitReview}

We give the most relevant work related to the paper and review some of the main developments in differentially private learning below.

\paragraph{Private Optimization:}
In~\cite{chaudhuri2011differentially}, Chaudhuri et al.\ studied
\emph{centralized} privacy-preserving machine learning algorithms for
convex optimization problem. In~\cite{bassily2014private},
Bassily et al.\ derived lower bounds on the empirical risk
minimization under \emph{central} differential privacy
constraints. Furthermore, they proposed a differential privacy SGD
algorithm that matches the lower bound for convex functions. In~\cite{abadi2016deep}, the
authors have generalized the private SGD algorithm proposed
in~\cite{bassily2014private} for non-convex optimization framework. In
addition, the authors have proposed a new analysis technique, called
moment accounting, to improve on the strong composition theorems to
compute the central differential privacy guarantee for iterative
algorithms. However, the works
mentioned,~\cite{chaudhuri2011differentially,bassily2014private,abadi2016deep},
assume that there exists a trusted server that collects the clients'
data. This motivates other works to design a distributed SGD
algorithms, where each client perturbs her own data without needing a
trusted server. 

Distributed learning under local differential privacy (LDP) has studied in~\cite{agarwal2018cpsgd,ESA,girgis2021shuffled-aistats}. In~\cite{agarwal2018cpsgd} the authors proposed a communication-efficient algorithm for learning models under local differential privacy. In~\cite{ESA}, the authors have proposed a
distributed local-differential-privacy gradient descent algorithm, a newly
proposed anonymization/shuffling framework
\cite{balle2019privacy} is used to amplify the privacy.  In~\cite{girgis2021shuffled-aistats}, the authors proposed communication efficient algorithms for general $\ell_p$-norm stetting under local differential privacy constraints, where they use recent results on
amplification by shuffling to boost the privacy-utility trade-offs of the distributed learning algorithms.

\paragraph{Shuffled privacy model:} The shuffled
model of privacy has been of significant recent interest
\cite{erlingsson2019amplification,ghazi2019power,balle2019improved,ghazi2019scalable,balle2019differentially,cheu2019distributed,balle2019privacy,balle2020private}. However, most of the existing works in 
literature~\cite{erlingsson2019amplification,balle2019privacy,feldman2020hiding}
only characterize the approximate DP of the shuffled model. Recently, the authors in~\cite{girgis2021renyi} proposed a novel bound on the RDP of the shuffled model, where they show that the RDP provides a significant saving in computing the total privacy budget for a composition of a sequence of shuffled mechanisms. However, the work~\cite{girgis2021renyi} does not characterize the RDP of the subsampled shuffle mechanism. We can compute a bound on the RDP of the subsampled shuffle mechanism by combining the bound of the RDP of the shuffle mechanism in~\cite{girgis2021renyi} with the bound of the subsampled RDP mechanism in~\cite{wang2019subsampled}. However, we show numerically that our new bound on the subsampled shuffle mechanism outperforms this bound.

\paragraph{Renyi differential privacy:} The work of Abadi
\emph{et al.} \cite{abadi2016deep} provided a new analysis technique to improve on the strong composition theorems. Inherently, this used Renyi divergence,
and was later formalized in \cite{mironov2017renyi} which defined
Renyi differential privacy (RDP). Several
works~\cite{mironov2019r,wang2019subsampled,zhu2019poission} have
shown that analyzing the RDP of subsampled mechanisms provides a
tighter bound on the total privacy loss than the bound that can be
obtained using the standard strong composition theorems. In
this paper, we analyze the RDP of the subsampled shuffle model, where we can
bound the approximate DP of a sequence of  subsampled shuffle models using the
transformation from RDP to approximate DP~\cite{abadi2016deep,wang2019subsampled,canonne2020discrete,asoodeh2021three}. We
show that our RDP analysis provides a better bound on the total
privacy loss of composition than the bound that can be obtained using the
standard strong composition theorems and the bound that can be obtained by combining the RDP bound of the shuffle model in~\cite{girgis2021renyi} with the subsampled RDP mechanism in~\cite{wang2019subsampled}.

%% file: app_ternary_special_case_DD.tex
%\section{Omitted Details from Section~\ref{sec:proof_ternary_DP_shuffle}}\label{app:ternary_DP_shuffle}

\section{Completing the Proof of Theorem~\ref{thm:ternary_DP_shuffle}}\label{subsec_app:completeting-proof}
For simplicity of notation, for any $m\in\{0,1,\hdots,n-1\}$, define 
\begin{align*}
q_m &:= \binom{k-1}{m} q^m(1-q)^{k-m-1} \\
E_m &:= \mathbb{E}_{\bh\sim\calM_{sh}(\calD_{m+1}''^{(k)})}\left[\left|\frac{\calM_{sh}(\calD_{m+1}^{(k)})(\bh)-\calM_{sh}(\calD_{m+1}'^{(k)})(\bh)}{\calM_{sh}(\calD_{m+1}''^{(k)})(\bh)}\right|^{\alpha}\right].
\end{align*} 
First we show an important property of $E_m$ that we will use in the proof.
\begin{lemma}\label{lem:E_m-decreasing}
$E_m$ is a non-increasing function of $m$, i.e., 
\begin{align}
&\mathbb{E}_{\bh\sim\calM_{sh}(\calD_{m+1}''^{(k)})}\left[\left|\frac{\calM_{sh}(\calD_{m+1}^{(k)})(\bh)-\calM_{sh}(\calD_{m+1}''^{(k)})(\bh)}{\calM_{sh}(\calD_{m+1}''^{(k)})(\bh)}\right|^{\alpha}\right] \notag \\
&\hspace{4cm}\leq \mathbb{E}_{\bh\sim\calM_{sh}(\calD_{m}''^{(k)})}\left[\left|\frac{\calM_{sh}(\calD_{m}^{(k)})(\bh)-\calM_{sh}(\calD_{m}'^{(k)})(\bh)}{\calM_{sh}(\calD_{m}''^{(k)})(\bh)}\right|^{\alpha}\right],\label{eq:E_m-decreasing}
\end{align}
where, for any $l\in\{m,m+1\}$, $\calD_{l}^{(k)}=(d''_k,\ldots,d''_k,d_{k})$, $\calD_{l}'^{(k)}=(d''_k,\ldots,d''_k,d'_{k})$, and  $\calD_{l}''^{(k)}=(d''_k,\ldots,d''_k,d''_k)$, each having $l$ elements.
\end{lemma}
\begin{proof}
Lemma~\ref{lem:E_m-decreasing} follows from Lemma~\ref{lem:cvx_tdp} in a straightforward manner, as Lemma~\ref{lem:cvx_tdp} is for arbitrary triples of adjacent datasets, whereas, Lemma~\ref{lem:E_m-decreasing} is for triples of adjacent datasets having special structures.
\end{proof}
Thus, we get
\begin{align}
\mathbb{E}_{\bh\sim\calM_{sh}(\calD'')}&\left[\left|\frac{\calM_{sh}(\calD)(\bh)-\calM_{sh}(\calD')(\bh)}{\calM_{sh}(\calD'')(\bh)}\right|^{\alpha}\right] \notag \\
&\hspace{3cm} \leq \sum_{m=0}^{k-1}q_mE_m \notag \\ 
&\hspace{3cm} =\sum_{m<\floor{(1-\gamma)q(k-1)}} q_m E_m+ \sum_{m\geq\floor{(1-\gamma)q(k-1)}} q_m E_m\notag \\
&\hspace{3cm} \stackrel{\text{(a)}}{\leq} E_{0}\sum_{m<\floor{(1-\gamma)q(k-1)}}q_m+\sum_{m\geq\floor{(1-\gamma)q(k-1)}} q_m E_m \notag \\
&\hspace{3cm} \stackrel{\text{(b)}}{\leq}  E_{0}e^{-\frac{q(k-1)\gamma^2}{2}}+\sum_{m\geq\floor{(1-\gamma)q(k-1)}} q_m E_m \notag \\
&\hspace{3cm} \stackrel{\text{(c)}}{\leq} e^{\epsilon_0\alpha}e^{-\frac{q(k-1)\gamma^2}{2}}+\sum_{m\geq\floor{(1-\gamma)q(k-1)}} q_m E_m \notag \\
&\hspace{3cm} \stackrel{\text{(d)}}{\leq} (e^{\epsilon_0}-e^{-\epsilon_0})^{\alpha}e^{-\frac{q(k-1)\gamma^2}{2}}+ E_{(1-\gamma)q(k-1)}. \label{proof_main-result_interim2}
\end{align}
Here, steps (a) and (d) follow from the fact that $E_m$ is a non-increasing function of $m$ (see Lemma~\ref{lem:E_m-decreasing}). Step (b) follows from the Chernoff bound. In step (c), we used that $\calM_{sh}(d_k)=\calR(d_k)$, $\calM_{sh}(d'_k)=\calR(d'_k)$, and $\calM_{sh}(d''_k)=\calR(d''_k)$ which together imply that $E_0=(e^{\epsilon_0}-e^{-\epsilon_0})^{\alpha}$, where the inequality follows because $\calR$ is an $\eps_0$-LDP mechanism. By choosing $\gamma=0.5$ completes the proof of Theorem~\ref{thm:ternary_DP_shuffle}.

%% file: app_sampling.tex
\section{Proof of Lemma~\ref{lem:samplig}}\label{app_sec:sampling}

Consider arbitrary neighboring datasets $\calD=\left(d_1,\ldots,d_n\right)\in\calX^{n}$ and $\calD'=\left(d_1,\ldots,d_{n-1},d_n'\right)\in\calX^{n}$. Recall that the LDP mechanism $\calR:\calX\to\calY$ has a discrete range $\calY=\left[B\right]$ for some $B\in\bbN$. Let $\bp_i:=(p_{i1},\ldots,p_{iB})$ and $\bp'_n:=(p'_{n1},\ldots,p'_{nB})$ denote the probability distributions over $\calY$ when the input to $\calR$ is $d_i$ and $d'_n$, respectively, where $p_{ij}=\Pr[\calR(d_i)=j]$ and $p_{nj}'=\Pr[\calR(d'_n)=j]$ for all $j\in[B]$ and $i\in\left[n\right]$. 

Let $\calP=\lbrace \bp_i:i\in\left[n\right] \rbrace$ and $\calP'=\lbrace \bp_i:i\in\left[n-1\right] \rbrace\bigcup \lbrace \bp'_n\rbrace$. For $i\in[n-1]$, let $\calP_{-i}=\calP\setminus\{\bp_i\}$, $\calP'_{-i}=\calP'\setminus\{\bp_i\}$, and also 
$\calP_{-n}=\calP\setminus\{\bp_n\}$, $\calP'_{-n}=\calP'\setminus\{\bp_n'\}$.

Here, $\calP,\calP'$ correspond to the datasets $\calD=\{d_1,\hdots,d_n\},\calD'=\{d_1,\hdots,d_{n-1},d'_n\}$, respectively, and for any $i\in[n]$, $\calP_{-i}$ and $\calP'_{-i}$ correspond to the datasets $\calD_{-i}=\{d_1,\hdots,d_{i-1},d_{i+1},\hdots,d_n\}$ and $\calD'_{-i}=\{d_1,\hdots,d_{i-1},d_{i+1},\hdots,d_{n-1},d'_n\}$, respectively. Thus, without loss of generality, we deal with sets $\calP$ and $\calP'$ throughout this section instead of dealing with $\calD$ and $\calD'$. Thus, we write $\calM\left(\calP\right)\triangleq \calM\left(\calD\right)$ and $\calM\left(\calP'\right)\triangleq \calM\left(\calD'\right)$.

We bound the Renyi divergence between $\calM(\calP)$ and $\calM(\calP')$. For given a set $\calS\subset[n]$ with $|\calS|=\gamma n = k$, we define two sets $\calP^{\calS},\calP^{'\calS}$, having $k$ distributions each, as follows:
\begin{align}
\calP^{\calS} &= \{\bp_i:i\in\calS\}, \label{eq:defn_P_C_S} \\
\calP^{'\calS} &= \{\bp_i:i\in\calS\}, \label{eq:defn_P_C_S-prime}
\end{align}
Observe that when $n\not\in\calS$, we have that $\calP^{\calS}=\calP^{'\calS}$. For given $k$ distributions $\calP=\left(\bp_1,\ldots,\bp_k\right)$, we define a shuffle mechanism $\calM_{sh}\left(\calP\right)$ as follows:
\begin{equation}
\calM_{sh}\left(\calP\right)=\calH_{k}\left(\bp_1,\ldots,\bp_k\right).
\end{equation}
Thus, the mechanisms $\calM(\calP)$ and $\calM(\calP')$ can be defined by:
\begin{align}
\calM(\calP)&=\frac{1}{\binom{n}{k}}\sum_{\calS\subset[n]}\calM_{sh}\left(\calP^{\calS}\right)\notag\\
&=\gamma P_E+(1-\gamma)P_{E^c}\\
\calM(\calP')&=\frac{1}{\binom{n}{k}}\sum_{\calS\subset[n]}\calM_{sh}\left(\calP^{'\calS}\right)\notag\\
&=\gamma Q_{E}+(1-\gamma)P_{E^c},
\end{align}
where $P_E = \frac{1}{\binom{n-1}{k-1}}\sum_{\substack{\calS\subset[n]\\ n\in\calS}}\calM_{sh}\left(\calP^{\calS}\right)$, $P_{E^c} = \frac{1}{\binom{n-1}{k}}\sum_{\substack{\calS\subset[n]\\ n\not\in\calS}}\calM_{sh}\left(\calP^{\calS}\right)$, and $Q_E = \frac{1}{\binom{n-1}{k-1}}\sum_{\substack{\calS\subset[n]\\ n\in\calS}}\calM_{sh}\left(\calP^{'\calS}\right)$. Hence, from the polynomial expansion, we get:
\begin{equation}
\begin{aligned}
\mathbb{E}_{\bh\sim \calM\left(\calP'\right)}\left[\left(\frac{\calM\left(\calP\right)\left(\bh\right)}{\calM\left(\calP'\right)\left(\bh\right)}\right)^{\lambda}\right]&=\mathbb{E}_{\bh\sim \calM\left(\calP'\right)}\left[\left(1+\frac{\calM\left(\calP\right)\left(\bh\right)}{\calM\left(\calP'\right)\left(\bh\right)}-1\right)^{\lambda}\right]\\
&=1+\sum_{j=2}^{\lambda} \binom{\lambda}{j} \mathbb{E}_{\bh\sim \calM\left(\calP'\right)}\left[\left(\frac{\calM\left(\calP\right)\left(\bh\right)-\calM\left(\calP'\right)\left(\bh\right)}{\calM\left(\calP'\right)\left(\bh\right)}\right)^{j}\right]\\
&=1+\sum_{j=2}^{\lambda} \binom{\lambda}{j} \gamma^{j} \mathbb{E}_{\bh\sim \calM\left(\calP'\right)}\left[\left(\frac{P_{E}\left(\bh\right)-Q_{E}\left(\bh\right)}{\calM\left(\calP'\right)\left(\bh\right)}\right)^{j}\right]~\label{eq:sampling_shuffle}
\end{aligned}
\end{equation}
Now, we borrow the trick used in~\cite{wang2019subsampled} to bound each term in the right hand side in~\eqref{eq:sampling_shuffle}. For completeness, we repeat their definitions and proofs here. We define an auxiliary dummy variable $i \sim \text{Unif}\left(1,\ldots,k\right)$ that is independent to everything else. Furthermore, we define two functions $g(\calS,i)$ and $g'(\calS,i)$ as follows:
\begin{align}
g(\calS,i)=\left\{\begin{array}{ll}
\calM_{sh}\left(\calP^{\calS}\right) & \text{if}\ n\in\calS\\
\calM_{sh}\left(\calP^{\calS\cup \lbrace n\rbrace\setminus \calS(i)}\right) & \text{otherwise}
\end{array}
\right.\\
g'(\calS,i)=\left\{\begin{array}{ll}
\calM_{sh}\left(\calP^{'\calS}\right) & \text{if}\ n\in\calS\\
\calM_{sh}\left(\calP^{'\calS\cup \lbrace n\rbrace\setminus \calS(i)}\right) & \text{otherwise}
\end{array}
\right.
\end{align}
Observe that $\mathbb{E}_{\calS,i}\left[g(\calS,i)\right]=P_{E}$, $\mathbb{E}_{\calS,i}\left[g'(\calS,i)\right]=Q_{E}$, and $\mathbb{E}_{\calS,i}\left[\calM_{sh}\left(\calP^{'\calS}\right)\right]=\calM\left(\calP\right)$. As a result, we get that
\begin{align*}
\mathbb{E}_{\bh\sim \calM\left(\calP'\right)}\left[\left(\frac{P_{E}\left(\bh\right)-Q_{E}\left(\bh\right)}{\calM\left(\calP'\right)\left(\bh\right)}\right)^{j}\right]&\leq \sum_{\bh}\frac{| P_{E}\left(\bh\right)-Q_{E}\left(\bh\right)|^{j}}{\left(\calM\left(\calP'\right)\left(\bh\right)\right)^{j-1}}\\
&\leq \sum_{\bh}\mathbb{E}_{\calS,i}\left[\frac{|g(\calS,i)\left(\bh\right) -g'(\calS,i)\left(\bh\right)|^{j}}{\left(\calM_{sh}\left(\calP^{'\calS}\right)\left(\bh\right) \right)^{j-1}}\right]\\
&=\mathbb{E}_{\calS,i}\mathbb{E}_{\bh\sim\calM_{sh}\left(\calP^{'\calS}\right)}\left[\left(\frac{|g(\calS,i)\left(\bh\right) -g'(\calS,i)\left(\bh\right)|}{\calM_{sh}\left(\calP^{'\calS}\right)\left(\bh\right)}\right)^{j}\right] \\
&\leq (\zeta(j))^{j}.
\end{align*}
Here, step (a) follows from Jensen's inequality and the convexity of the function $x^j/y^{j-1}$ (see~\cite[Lemma~$20$]{wang2019subsampled}). Step (b) follows from Fubini's theorem. The last inequality is obtained by taking the supremum over all possible neighboring datasets $\calD$, $\calD'$. This completes the proof of Lemma~\ref{lem:samplig}.